\documentclass[11pt]{article}
\usepackage{amsthm}
\usepackage{tgtermes}
\usepackage{newtxtext}
\usepackage{newtxmath}
\usepackage{bm}

\usepackage[T1]{fontenc}
\usepackage[letterpaper,margin=1in]{geometry}
\usepackage{mathtools}
\usepackage{graphicx,booktabs,array,tabularx}
\usepackage{longtable,etoolbox}
\usepackage{caption}
\DeclareCaptionFont{benchmarkcaption}{\fontsize{10}{12}\selectfont}
\usepackage{float}
\usepackage{placeins}
\usepackage{tikz}
\usetikzlibrary{arrows.meta,positioning,calc}
\usepackage[round,authoryear]{natbib}
\usepackage{microtype}
\usepackage[colorlinks=true,linkcolor=blue!45!black,citecolor=blue!45!black,urlcolor=blue!45!black]{hyperref}

\AtBeginEnvironment{table}{\TableSpaced}
\AtBeginEnvironment{figure}{\TableSpaced}
\AtBeginEnvironment{longtable}{\TableSpaced}
\floatstyle{ruled}
\newfloat{algorithm}{tbp}{loa}
\floatname{algorithm}{Algorithm}
\newcommand{\B}{\mathcal B}
\newcommand{\Lset}{\mathcal L}
\newcommand{\F}{\mathcal F}
\newcommand{\K}{\mathcal K}
\newcommand{\Q}{\mathcal Q}
\newcommand{\Tset}{\mathcal T}
\newcommand{\Om}{\Omega}
\newcommand{\splitact}{\operatorname{split}}
\newcommand{\predictact}{\operatorname{predict}}
\newcommand{\conv}{\operatorname{conv}}
\newcommand{\OPT}{\mathrm{OPT}}
\newcommand{\LB}{\mathrm{LB}}
\newcommand{\UB}{\mathrm{UB}}

\newcommand{\papertitle}{An Exact Junction-Tree Extended Formulation for Optimal Classification Trees}

\newcommand{\TableSpaced}{}
\newcommand{\OneAndAHalfSpacedXI}{}
\newcommand{\EquationsNumberedThrough}{}
\newcommand{\TheoremsNumberedThrough}{}
\newcommand{\ECRepeatTheorems}{}
\newcommand{\MANUSCRIPTNO}[1]{}
\newcommand{\RUNTITLE}[1]{}
\newcommand{\RUNAUTHOR}[1]{}
\newcommand{\TITLE}[1]{\title{#1}}
\newcommand{\ARTICLEAUTHORS}[1]{\author{%
Jiancheng Tu\\
\small Department of Computing, The Hong Kong Polytechnic University\\
\small \texttt{jiancheng.tu@connect.polyu.hk}
\and
Wenqi Fan\\
\small Department of Computing, The Hong Kong Polytechnic University\\
\small Department of Management and Marketing, The Hong Kong Polytechnic University\\
\small \texttt{wenqi.fan@polyu.edu.hk}
}\date{}}
\newcommand{\ABSTRACT}[1]{\gdef\paperabstract{#1}}
\newcommand{\KEYWORDS}[1]{\gdef\paperkeywords{#1}}
\newcommand{\ECSwitch}{}
\newcommand{\ECHead}[1]{\clearpage\section*{#1}\setcounter{section}{0}\renewcommand{\thesection}{EC.\arabic{section}}}
\newtheorem{theorem}{Theorem}

\newtheorem{proposition}{Proposition}
\newtheorem{corollary}{Corollary}
\theoremstyle{definition}
\newtheorem{definition}{Definition}
\theoremstyle{remark}
\newtheorem{remark}{Remark}
\OneAndAHalfSpacedXI
\EquationsNumberedThrough
\TheoremsNumberedThrough
\ECRepeatTheorems
\MANUSCRIPTNO{}
\begin{document}
\RUNTITLE{Junction-Tree Formulation for Optimal Classification Trees}
\RUNAUTHOR{Anonymous authors}
\TITLE{\papertitle}
\ARTICLEAUTHORS{}

% ===== BEGIN inlined file: sections/abstract =====
\ABSTRACT{%
{
We develop an exact linear programming (LP) formulation for bounded-depth
classification trees with binary features, using a junction-tree representation.
The formulation is integral and supports recursive subtree optimization.
Exact reductions make the model smaller while preserving the optimal value
and recovery of an optimal tree. The reduced model supports two solution
methods: column generation and message passing. Column generation solves
integral restricted LPs and uses bounds over the full feasible domain to
certify optimality. Message passing recursively combines optimal subtree
costs. Both methods solve common subtree problems that, once the preceding
tree decisions are fixed, can be evaluated independently and in parallel.
Computational
experiments show that the exact reductions substantially reduce the size of
the junction-tree formulation. The resulting linear programming formulation
certifies instances for which the tested mixed-integer formulation does not
establish optimality within the same computational budget, while the
column-generation and message-passing methods certify more instances and
achieve an order-of-magnitude reduction in geometric-mean runtime relative
to an existing state-of-the-art exact method for optimal classification
trees.
\par}}
% ===== END inlined file: sections/abstract =====

\maketitle

\begin{abstract}
\paperabstract
\end{abstract}


% ===== BEGIN inlined file: body =====
% ===== BEGIN inlined file: sections/introduction =====

\section{Introduction}\label{sec:intro}

Classification trees are widely used when predictive performance must be combined with interpretable decision rules. Classical methods such as classification and regression trees (CART) construct a tree greedily, choosing each split according to its immediate improvement \citep{breiman1984cart}. Because an upstream split determines the observations available to all downstream decisions, such local choices need not produce the best tree under a prescribed global objective. Optimizing the tree jointly avoids this limitation, but constructing an optimal binary decision tree is NP-hard \citep{hyafil1976}.

Exact optimization methods for classification trees must address two related issues. First, they must represent globally compatible splitting, prediction, and stopping decisions. Second, they must exploit enough structure in that representation to search or coordinate the resulting decision space efficiently and to certify optimality. Existing work approaches these issues through mathematical and constraint-based formulations, as well as through dynamic programming and exact search. For a broader review of optimization models for classification and regression trees, see \citet{carrizosa2021review}.

\subsection{Optimal Classification Trees}
\label{sec:intro:exactmethods}

Mixed-integer optimization provides a general framework for jointly choosing
splits, routing observations, and assigning predictions in classification
trees. \citet{bertsimas2017oct} formulate these decisions through an explicit
tree template with observation-to-leaf assignments, while
\citet{verwer2019blp} develop a binary linear formulation that reduces the
dependence on the number of candidate split values.
\citet{gunluk2021categorical} exploit the structure of categorical features in
an integer programming formulation. 
Optimal trees have also been
modeled through Boolean satisfiability (SAT) and maximum satisfiability (MaxSAT) encodings
\citep{narodytska2018sat,hu2020maxsat} and through constraint programming
\citep{verhaeghe2020cp}. Although these formulations offer considerable
modeling flexibility, their computational burden grows with the number of observations and tree depth.

Subsequent work has strengthened these formulations and developed decomposition
methods. \citet{aghaei2024strong} introduce a flow-based formulation with a
stronger linear programming (LP) relaxation and exploit its structure through Benders decomposition.
\citet{ales2024new,alston2026milp} develop alternative flow-, cut-based, and
compact formulations. \citet{michini2024polyhedral} study the polyhedral
structure of realizable routings for multivariate trees and derive inequalities
that strengthen the corresponding mixed-integer formulations.
For depth-two trees, \citet{organ2026rolling} derive a formulation whose LP relaxation describes the convex hull of feasible trees.

Path-based formulations provide a different representation and naturally lead
to column generation \citep{firat2020cg, patel2024cg, subramanian2023multiway}. \citet{firat2020cg} represent root-to-leaf decision paths
as columns of an integer master formulation and generate promising paths
through pricing. \citet{patel2024cg} strengthen this approach through improved
pricing, preprocessing, and additional valid inequalities.
\citet{subramanian2023multiway} develop a path-based formulation for
constrained multiway-split decision trees and use column generation to handle
the large path space.
A limitation of these
approaches is that column generation solves only the LP relaxation of the
underlying integer path formulation, so convergence does not by itself certify
optimality for the original tree problem.

Specialized dynamic-programming and exact-search methods have substantially
improved the scalability of optimal decision-tree optimization relative to
general-purpose solver-based formulations \citep{aglin2020dl85, demirovic2022murtree, vanderlinden2023}. DL8.5 \citep{aglin2020dl85},
MurTree \citep{demirovic2022murtree}, and STreeD
\citep{vanderlinden2023} exploit subtree decomposability, caching, and
problem-specific bounds to reuse equivalent subproblems. Optimal sparse decision trees (OSDT) and generalized and scalable optimal sparse decision trees (GOSDT)
\citep{hu2019osdt,lin2020gosdt} instead organize the search around sparse-tree
representations and specialized lower bounds, while Branches
\citep{chaouki2025} uses an AND/OR graph representation to structure the
search.  Their
scalability, however, remains limited by the combinatorial growth of the state
and search spaces: worst-case complexity grows exponentially with tree size and
the number of candidate binary features, so performance can deteriorate rapidly
for deeper trees or high-dimensional feature sets. Related exact methods such
as Quant-BnB \citep{mazumder2022quantbnb} and ConTree
\citep{brita2025contree} extend branch-and-bound and dynamic-programming ideas
to continuous features.

\subsection{Discussion}
\label{sec:intro:discussion}

The preceding literature offers complementary ways to exploit tree structure.
Path-based column generation searches a large path space through the LP
relaxation of an integer formulation. Recursive exact-search methods use
subtree decomposability, cached states, and bounds to avoid repeated search.
The integral shallow formulations of \citet{organ2026rolling} provide a
polyhedral starting point. We seek an exact LP representation for general
prescribed depth that also allows smaller subtrees to be optimized and
combined recursively.

Our approach draws on junction-tree theory to connect subtree optimization
with an exact LP representation \citep{wainwright2008graphical,kolman2015},
and on the relationship between dynamic programming and linear optimization
\citep{martin1990}. The same decomposition supports recursive evaluation
by message passing \citep{kschischang2001factor,wainwright2008graphical}.

For optimal classification trees, the central modeling issue is that earlier splits determine which
observations reach a subtree and which subsequent decisions are feasible.
We represent these dependencies in an LP that allows subtrees to be
optimized separately and their solutions combined into a feasible
classification tree. This provides the basis for reducing the model and
organizing the computation around subtree optimization.

\subsection{Proposed Approach and Contributions}
\label{sec:intro:contributions}

Our approach combines an exact LP representation of classification trees
with algorithms that optimize and combine smaller subtrees. We make three
contributions.

First, we give an exact LP formulation for classification trees of any
prescribed maximum depth with binary features. The model accommodates early
stopping and restrictions such as minimum leaf support. We establish a
convex-hull representation from which an optimal tree can be recovered by
solving the LP. Its variables and constraints describe tree decisions,
without introducing a separate set for each training observation. The data
enter through the costs and feasibility of the split and prediction choices.

Second, we derive exact reductions that make the formulation smaller.
They merge choices with the same effect on the rest of the tree, optimize
subtrees separately once the preceding decisions are fixed, and eliminate
variables whose contributions can be incorporated into the remaining model.
These operations preserve the optimal value and recovery of an optimal tree.
They reduce both the number of explicit choices and the coordination needed
to combine subtree decisions.

Third, we develop two exact solution methods for the reduced model.
Junction-tree column generation (JT-CG) solves a sequence of smaller LPs,
adding promising combinations of tree decisions and using bounds over all
feasible trees to certify optimality. The restricted LPs remain integral.
Junction-tree message passing (JT-MP) optimizes and combines subtrees
recursively to recover an optimal tree. Both methods solve the same subtree
problems, which can be evaluated independently and in parallel once the
preceding tree decisions are fixed. Experiments show gains in
 runtime relative to the tested mixed-integer programming
and exact-search methods.

The remainder of the paper develops the junction-tree formulation and its
polyhedral properties in Section~\ref{sec:formulation}, the exact reductions in
Section~\ref{sec:exact_reductions}, and the solution methods in
Section~\ref{sec:methods}. Section~\ref{sec:experiments} reports the
computational study, and Section~\ref{sec:conclusions} concludes.

\section{Junction-Tree Formulation}\label{sec:formulation}

We first specify the bounded-depth optimal classification tree (OCT) model and then construct a junction-tree
configuration formulation whose polyhedral properties are developed below.

\subsection{Problem Setup}\label{sec:jtform}

We consider training observations with binary features. Let
$I=\{1,\ldots,n\}$ index the observations, let $\F=\{1,\ldots,F\}$ index the
available features, and let $\K$ be the set of class labels, with $K=|\K|$. Observation
$i\in I$ has feature vector $z_i\in\{0,1\}^{F}$ and label $y_i\in\K$.
For a prescribed maximum depth $D\geq1$, we use the complete binary-tree
template
\begin{equation}\label{eq:template}
 \B=\bigcup_{h=0}^{D-1}\{0,1\}^h,
 \qquad
 \Lset=\{0,1\}^D.
\end{equation}
A binary string identifies a position in the template: $\epsilon$ denotes the
root, and appending $0$ or $1$ gives the left or right child. The sets $\B$
and $\Lset$ contain the potential internal and terminal positions,
respectively. Figure~\ref{fig:jt}(a) illustrates this notation for $D=3$.
For example, the children of position $0$ are $00$ and $01$, and the terminal
children of $00$ are $000$ and $001$.

A classification tree assigns an action $T_v$ to every position
$v\in\B\cup\Lset$. At an internal position, the action is either a split
$\splitact(f)$ on some feature $f\in\F$, a prediction $\predictact(k)$ of some
class $k\in\K$, or the inactive action $\bot$; at a terminal position, the
action is a prediction or $\bot$. The root is active, and a child is active
if and only if its parent splits. Hence a prediction terminates its branch
and makes all descendants inactive. The action $\bot$ marks positions of the
complete template that do not belong to the selected tree.

The edge labels in Figure~\ref{fig:jt}(a) describe observation routing. If
position $v$ splits on feature $f$, observation $i$ follows the edge labeled
$0$ when $z_{if}=0$ and the edge labeled $1$ when $z_{if}=1$. Thus, the
actions along a root-to-leaf path determine both the selected tree and the
observations reaching each prediction leaf. Early stopping is represented by
placing a prediction at an internal position and assigning $\bot$ to its
descendants.

% ===== BEGIN inlined file: figures/junction_tree =====
\begin{figure}[!htbp]
\centering
\begin{tikzpicture}[x=0.96cm,y=0.88cm,
  split/.style={circle,draw,minimum size=7mm,inner sep=1pt,font=\small},
  leaf/.style={rectangle,draw,minimum width=9mm,minimum height=5mm,
    inner sep=1pt,font=\scriptsize},
  route/.style={->,>=stealth,thin},
  shared/.style={draw=blue!65!black,fill=blue!8,line width=0.9pt}]

  \node[anchor=west,font=\small\bfseries] at (-0.3,0.8)
    {(a) Classification-tree template, $D=3$};

  \node[split,shared] (r) at (7,0) {$\epsilon$};
  \node[split,shared] (v0) at (3,-1.1) {$0$};
  \node[split] (v1) at (11,-1.1) {$1$};

  \node[split] (v00) at (1,-2.2) {$00$};
  \node[split] (v01) at (5,-2.2) {$01$};
  \node[split] (v10) at (9,-2.2) {$10$};
  \node[split] (v11) at (13,-2.2) {$11$};

  \node[leaf] (l000) at (0.25,-3.3) {$000$};
  \node[leaf] (l001) at (1.75,-3.3) {$001$};
  \node[leaf] (l010) at (4.25,-3.3) {$010$};
  \node[leaf] (l011) at (5.75,-3.3) {$011$};
  \node[leaf] (l100) at (8.25,-3.3) {$100$};
  \node[leaf] (l101) at (9.75,-3.3) {$101$};
  \node[leaf] (l110) at (12.25,-3.3) {$110$};
  \node[leaf] (l111) at (13.75,-3.3) {$111$};

  \draw[route,blue!65!black,line width=1pt]
    (r)--node[above,font=\scriptsize] {$0$} (v0);
  \draw[route]
    (r)--node[above,font=\scriptsize] {$1$} (v1);

  \draw[route,blue!65!black,line width=1pt]
    (v0)--node[above,font=\scriptsize] {$0$} (v00);
  \draw[route,blue!65!black,dashed,line width=1pt]
    (v0)--node[above,font=\scriptsize] {$1$} (v01);

  \draw[route]
    (v1)--node[above,font=\scriptsize] {$0$} (v10);
  \draw[route]
    (v1)--node[above,font=\scriptsize] {$1$} (v11);

  \foreach \p/\l/\rr in
    {v00/l000/l001,v01/l010/l011,v10/l100/l101,v11/l110/l111}{
    \draw[route] (\p)--node[left,font=\scriptsize] {$0$} (\l);
    \draw[route] (\p)--node[right,font=\scriptsize] {$1$} (\rr);
  }

  \node[font=\scriptsize,align=center] at (7,-4.0)
    {The highlighted paths are associated with
     clusters $C_{00}$ and $C_{01}$.};
\end{tikzpicture}

\medskip

\begin{tikzpicture}[x=1cm,y=1cm,
  cluster/.style={rectangle,rounded corners=2pt,draw,align=center,
    minimum width=3.25cm,minimum height=1.25cm,
    inner sep=3pt,font=\scriptsize}]

  \node[anchor=west,font=\small\bfseries] at (-1.43,1.41)
    {(b) Junction-tree representation};

  \node[cluster,fill=blue!5] (c00) at (0,0)
    {$C_{00}$\\$\{\epsilon,0,00,000,001\}$};
  \node[cluster,fill=blue!5] (c01) at (4,0)
    {$C_{01}$\\$\{\epsilon,0,01,010,011\}$};
  \node[cluster] (c10) at (8,0)
    {$C_{10}$\\$\{\epsilon,1,10,100,101\}$};
  \node[cluster] (c11) at (12,0)
    {$C_{11}$\\$\{\epsilon,1,11,110,111\}$};

  \draw[thick] (c00)--(c01);
  \draw[thick] (c01)--(c10);
  \draw[thick] (c10)--(c11);

  \node[font=\scriptsize,fill=white,inner sep=1pt] at (2,0.955)
    {$S=\{\epsilon,0\}$};
  \node[font=\scriptsize,fill=white,inner sep=1pt] at (6,0.955)
    {$S=\{\epsilon\}$};
  \node[font=\scriptsize,fill=white,inner sep=1pt] at (10,0.955)
    {$S=\{\epsilon,1\}$};
\end{tikzpicture}

\caption{Depth-three classification-tree template and its junction-tree
representation. Panel~(a) shows the binary-string position labels and the
$0$--$1$ routing convention. Panel~(b) groups each sibling terminal pair
with its complete ancestor path. The intersections displayed above the
edges are the separators of the junction tree.}
\label{fig:jt}
\end{figure}
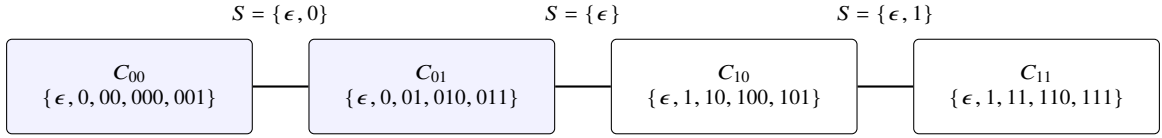
% ===== END inlined file: figures/junction_tree =====

Let $B(T)$ denote the split nodes of $T$ and $L(T)$ its prediction leaves.
At a split node $v$, let $f_v$ denote the selected feature. Each observation
follows the induced root-to-leaf path, and $T(z_i)$ denotes its predicted
class. We consider the objective
\begin{equation}\label{eq:objective}
 J(T)=\frac{1}{n}\sum_{i\in I}\mathbf 1\{T(z_i)\neq y_i\}
 +\sum_{v\in B(T)}\lambda_{v,f_v},
 \qquad \lambda_{v,f}\geq0.
\end{equation}
The first term is the training misclassification rate and the second
penalizes split decisions. A uniform split penalty gives $\lambda|B(T)|$, as
in standard OCT regularization
\citep{bertsimas2017oct,aghaei2024strong}. When $\lambda=0$, minimizing
\eqref{eq:objective} is equivalent to maximizing training accuracy.

Let $\Tset$ denote the set of admissible trees satisfying the structural
rules above and any prescribed path or leaf requirements. For each position
$v$, let $A_v$ denote its strict ancestors. We assume that feasibility is
determined by the parent--child rules and by predicates depending only on
$(T_u:u\in A_v\cup\{v\})$ and the fixed training data. This setting includes
path restrictions such as prohibiting repeated features and leaf requirements
such as minimum support. Constraints coupling decisions in different branches
require an augmented state representation. 
The optimal classification tree problem is
\begin{equation}\label{eq:oct}
 \OPT_D=\min\{J(T):T\in\Tset\}.
\end{equation}

\subsection{Junction-Tree Construction}\label{sec:jtconstruction}

We now group overlapping subsets of the template into the junction-tree
clusters illustrated in Figure~\ref{fig:jt}(b).

\begin{definition}[Junction tree; \citealp{wainwright2008graphical}]
\label{def:junctiontree}
Let $\{C_q:q\in\mathcal Q\}$ be a family of sets, called clusters. A tree
$\mathcal J=(\mathcal Q,E)$ is a junction tree if, for every template
position $v$ contained in at least one cluster, the set
\[
 \{q\in\mathcal Q:v\in C_q\}
\]
induces a connected subtree of $\mathcal J$. For an edge
$e=\{q,r\}\in E$, the intersection
\[
 S_e=C_q\cap C_r
\]
is its separator.
\end{definition}

The connectedness condition is the running-intersection property: if two
clusters contain the same template position, every cluster on the path
between them in $\mathcal J$ also contains that position.

Let
\[
 \Q=\{0,1\}^{D-1}
\]
be the set of parents of the terminal positions. For each $q\in\Q$, define
\begin{equation}\label{eq:clusterdef}
 C_q=
 \{q0,q1\}
 \cup
 \{u:u\text{ is a strict ancestor of }q\}
 \cup
 \{q\}.
\end{equation}
Thus, $C_q$ contains the sibling terminal positions $q0$ and $q1$ together
with their complete ancestor paths. The cluster family is fixed by the
complete template and does not depend on the selected tree. If a branch
terminates before depth $D$, positions below the prediction remain in the
cluster but receive the inactive action $\bot$.

A convenient junction tree for this cluster family is obtained by ordering
$\Q$ lexicographically and connecting consecutive clusters. We denote the
resulting chain by $\mathcal J$. Its edges describe overlaps between
clusters; they are not parent--child edges of the classification tree.

For $D=3$,
\[
 \Q=\{00,01,10,11\},
 \qquad
 C_{00}-C_{01}-C_{10}-C_{11}.
\]
For example,
\[
 C_{00}=\{\epsilon,0,00,000,001\},
 \qquad
 C_{01}=\{\epsilon,0,01,010,011\},
\]
so
\[
 C_{00}\cap C_{01}=\{\epsilon,0\}.
\]
Similarly,
\[
 C_{01}\cap C_{10}=\{\epsilon\},
 \qquad
 C_{10}\cap C_{11}=\{\epsilon,1\},
\]
as shown in Figure~\ref{fig:jt}(b).

\begin{proposition}[Validity of the cluster chain]
\label{prop:jtconstruction}
The chain $\mathcal J$ constructed above satisfies the running-intersection
property and hence is a junction tree for the cluster family
$\{C_q:q\in\Q\}$.
\end{proposition}

For any internal position, the clusters containing it form a consecutive
lexicographic interval; terminal positions each occur in a single cluster.
This establishes running intersection, as detailed in Electronic
Companion~\ref{ec:chainproofsection}. We can now specify the decisions
carried by each cluster.

\begin{definition}[Local configuration]\label{def:configuration}
A local configuration on cluster $C_q$ assigns an action to every position
in $C_q$. Internal positions may split, predict, or be inactive, whereas the
two terminal positions $q0$ and $q1$ may predict or be inactive. A
configuration is admissible if its actions satisfy the structural rules and
the ancestor-local feasibility requirements at all positions in $C_q$. The
finite set of admissible configurations for cluster $q$ is denoted by
$\Om_q$.
\end{definition}

Because $C_q$ contains the complete ancestor paths of its terminal positions,
all ancestor-local requirements within the cluster can be evaluated from the
configuration itself. For example, a configuration in $\Om_{00}$ may split
at $\epsilon$, $0$, and $00$ and predict at $000$ and $001$. A prediction
at $00$ makes $000$ and $001$ inactive, while a prediction at $0$ makes
$00$, $000$, and $001$ inactive.

For an edge $e=\{q,r\}\in E$, configurations on $C_q$ and $C_r$ are
consistent if they assign the same action to every position in $S_e$. Thus,
in the depth-three example, configurations on $C_{00}$ and $C_{01}$ must
agree jointly on the actions at $\{\epsilon,0\}$. These separator agreements
become the coupling constraints of the formulation below.

\subsection{Configuration Formulation}\label{sec:jtmodel}

The formulation selects one local configuration for each cluster and
coordinates neighboring clusters through their separator assignments.
Training loss and split penalties enter through local configuration costs.

A template position can appear in several clusters. Let
\[
 m_v=\left|\{q\in\Q:v\in C_q\}\right|
\]
be the number of clusters containing position $v$. For the construction
above,
\[
 m_v=2^{D-1-|v|}
\]
for an internal position $v$, and $m_v=1$ for a terminal position, where
$|v|$ denotes the depth of $v$. We divide the cost incurred at position $v$
equally among these $m_v$ clusters.
Because each cluster contains the complete ancestor path of every position it
contains, a configuration $\omega\in\Om_q$ determines which observations
reach each position in $C_q$. Define
\begin{equation}\label{eq:positioncost}
 \ell_v(\omega)=
 \begin{cases}
  \lambda_{v,f}, & \omega_v=\splitact(f),\\[1mm]
  \displaystyle
  \frac{1}{n}\sum_{\substack{i\in I:\\
          i\text{ reaches }v\text{ under }\omega}}
          \mathbf 1\{y_i\neq k\},
      & \omega_v=\predictact(k),\\[3mm]
  0, & \omega_v=\bot.
 \end{cases}
\end{equation}
The local configuration cost is
\begin{equation}\label{eq:localcost}
 c_q(\omega)=\sum_{v\in C_q}\frac{\ell_v(\omega)}{m_v},
 \qquad q\in\Q,\quad\omega\in\Om_q.
\end{equation}
For the depth-three example, the root cost is divided among four clusters,
whereas the cost at position $0$ is divided between $C_{00}$ and $C_{01}$.
A terminal prediction cost is assigned to its unique cluster.

\begin{proposition}[Cost decomposition]\label{prop:costdecomp}
For every $T\in\Tset$,
\begin{equation}\label{eq:costidentity}
 J(T)=\sum_{q\in\Q}c_q(T|_{C_q}).
\end{equation}
\end{proposition}

The identity follows because the cost at each position $v$ is divided among
exactly the $m_v$ clusters containing that position. A proof is given in
Electronic Companion~\ref{ec:hull}.

For every $q\in\Q$ and $\omega\in\Om_q$, let $x_{q\omega}=1$ if
configuration $\omega$ is selected in cluster $q$, and let it be zero
otherwise. For an edge $e=\{q,r\}\in E$, let $\Sigma_e$ be the set of
separator assignments occurring in a configuration of either endpoint, and
write $\omega|_{S_e}$ for the restriction of $\omega$ to $S_e$. The
junction-tree integer  programming formulation is
\begin{subequations}\label{eq:jtip}
\begin{align}
 \min\quad
 &\sum_{q\in\Q}\sum_{\omega\in\Om_q}c_q(\omega)x_{q\omega},
 \label{eq:jtipobj}\\
 \text{s.t.}\quad
 &\sum_{\omega\in\Om_q}x_{q\omega}=1
 && q\in\Q,
 \label{eq:jtnorm}\\
 &\sum_{\substack{\omega\in\Om_q:\\ \omega|_{S_e}=\sigma}}x_{q\omega}
 -\sum_{\substack{\omega\in\Om_r:\\ \omega|_{S_e}=\sigma}}x_{r\omega}=0
 && e=\{q,r\}\in E,\ \sigma\in\Sigma_e,
 \label{eq:jtsep}\\
 &x_{q\omega}\in\{0,1\}
 && q\in\Q,\ \omega\in\Om_q.
 \label{eq:jtbinary}
\end{align}
\end{subequations}
Constraint~\eqref{eq:jtnorm} selects one configuration in each cluster.
Constraint~\eqref{eq:jtsep} matches complete separator assignments between
neighboring clusters.
For a feasible tree $T\in\Tset$, define its configuration-selection vector
by
\begin{equation}\label{eq:incidence}
 x^T_{q\omega}=\mathbf 1\{T|_{C_q}=\omega\}.
\end{equation}
We refer to \eqref{eq:jtip} as the junction-tree integer program (JT-IP). Its LP relaxation, the junction-tree LP (JT-LP), replaces
\eqref{eq:jtbinary} with $x_{q\omega}\geq0$; let $P$ denote the resulting
feasible region.

\subsection{Polyhedral Properties}\label{sec:polyhedral}

JT-LP is a configuration formulation on a valid junction tree. Classical
junction-tree consistency \citep[Proposition~2.1]{wainwright2008graphical}
and bounded-treewidth configuration results
\citep[Lemma~2.1 and Theorem~1.1]{kolman2015} provide the generic
integrality mechanism. The OCT construction above specifies the local
actions, clusters, separator assignments, and costs to which these results
are applied. The remaining question is whether consistency of these local
actions enforces all OCT requirements. Complete ancestor paths ensure that
each routing and feasibility condition is checked with the same information
in every cluster containing it. The next theorem combines this correspondence
with junction-tree consistency; Proposition~\ref{prop:costdecomp} then transfers
the polyhedral result to the OCT objective.

\begin{theorem}[Convex-hull representation]\label{thm:integrality}
Under the ancestor-local feasibility assumptions of
Section~\ref{sec:jtform}, separator-consistent local configurations are in
one-to-one correspondence with the feasible template trees in $\Tset$.
Consequently,
\begin{equation}\label{eq:hull}
 P=\conv\{x^T:T\in\Tset\}.
\end{equation}
Hence every extreme point of JT-LP is integral. If
$\Tset\neq\varnothing$, JT-IP and JT-LP have optimal value $\OPT_D$, and an
optimal classification tree can be recovered from an optimal JT-LP
solution.
\end{theorem}

Electronic Companion~\ref{ec:hull} proves the configuration--tree
correspondence and \eqref{eq:hull}. The constructive proof also gives a
procedure for recovering an optimal tree from an optimal LP solution.
Fixing excluded nonnegative configuration variables to zero defines a face
of $P$. This observation matters when the full configuration domain is too
large to enumerate: retaining a compatible subset of configurations preserves
integrality even though it can exclude the globally optimal tree. The next
corollary establishes the feasible-tree interpretation of the restricted
masters used by column generation.

\begin{remark}[Shallow configuration models]\label{rem:relatedoct}
For $D=2$, the formulation reduces to the configuration structure of
\citet{organ2026rolling} under matched feature, loss, and feasibility
assumptions. Under the corresponding depth-three assumptions, the four-group
formulation in Appendix~A of the 2023 preprint of \citet{organ2026rolling} has the same
normalization and joint-ancestor consistency structure. Electronic
Companion~\ref{ec:hull} gives the detailed variable mapping. The construction
above extends this configuration principle to arbitrary prescribed depth
with early stopping and ancestor-local feasibility.
\end{remark}

\begin{corollary}[Restricted-master integrality]\label{cor:rmp}
For arbitrary subsets $\widehat\Om_q\subseteq\Om_q$, restrict JT-LP to the
variables associated with $\widehat\Om_q$ and retain every separator equation
\eqref{eq:jtsep} whose assignment occurs in a retained configuration at
either endpoint. After extending each restricted vector by zeros on all
excluded coordinates, the resulting feasible region is either empty or
\[
 \conv\left\{
 x^T:T\in\Tset,\;
 T|_{C_q}\in\widehat\Om_q\text{ for every }q\in\Q
 \right\}.
\]
Consequently, every nonempty restricted master has only integral extreme
points and an integral optimum for every linear objective.
\end{corollary}

A feasible restricted master therefore yields a feasible classification tree
and an upper bound on $\OPT_D$. Certification for the full problem requires
the pricing or lower-bound tests developed in Section~\ref{sec:cg}.

\begin{remark}[Extensions of the model]
\label{rem:linearobjectives}
Any objective that is additive in fixed prediction and split costs changes
only $c_q(\omega)$ and leaves the feasible region and
Theorem~\ref{thm:integrality} unchanged. For example, balanced
misclassification loss is obtained by replacing the weight $1/n$ in
\eqref{eq:positioncost} with $1/(|\K|n_{y_i})$, where $n_k>0$ is the number
of observations in class $k$.
The construction also extends to multiway classification trees on a fixed
bounded-depth template. Each split rule specifies a partition of the
observations among its branches. When costs remain additive and feasibility
is determined by a decision and its ancestors, the same consistency and
tree-recovery arguments apply after adapting the local configurations to
the multiway template. The formulation sizes depend on the branching
structure.
\end{remark}

Each cluster contains $D+2$ template positions, so the induced tree
decomposition has width at most $D+1$. With explicit prediction labels, no
feasibility filtering, and early stopping allowed, the unfiltered number of
configurations per cluster is
$
 F^D K^2+K\sum_{j=0}^{D-1}F^j$.
Since there are $2^{D-1}$ clusters, the corresponding total number of
configuration variables is
$
 2^{D-1}
 \left(
 F^D K^2+K\sum_{j=0}^{D-1}F^j
 \right)$.

For comparison, consider the full-depth split-only representation in which
every internal position splits and terminal predictions are optimized
independently rather than indexed as configuration actions. In this case the
displayed formulation has
\begin{equation}\label{eq:sizesplit}
 \underbrace{2^{D-1}F^D}_{\text{columns}},
 \qquad
 \underbrace{2^{D-1}+\sum_{h=1}^{D-1}2^{h-1}F^h}_{\text{equalities}},
\end{equation}
before redundant rows are removed. The master contains no
observation-indexed variables or constraints, although its local costs and
feasibility filters depend on the training data. The configuration space nevertheless grows
exponentially with the prescribed depth. This growth motivates the exact
reductions developed in Section~\ref{sec:exact_reductions}.

\section{Exact Reductions}\label{sec:exact_reductions}

All three reductions use the same principle: once the assignments on the
relevant separators are fixed, decisions that are not visible outside those
separators can be optimized locally. Separator-signature compression applies
this principle within a cluster, subtree contraction applies it to a block of
clusters, and endpoint elimination removes a leaf cluster after transferring
its conditional cost to its neighbor. For the fixed additive objective
\eqref{eq:objective}, each operation preserves the optimal value and retains
enough information to recover an optimal tree. Proofs are given in Electronic
Companion~\ref{ec:exactreductions}.

\subsection{Signature Compression}\label{sec:signature_compression}

For a cluster $q\in\Q$, let $N_{\mathcal J}(q)$ denote its neighbors in
$\mathcal J$. The separator signature of a configuration
$\omega\in\Om_q$ is
\begin{equation}\label{eq:separator_signature}
 \gamma_q(\omega)=
 \bigl(\omega|_{S_e}:e=\{q,r\}\in E,\ r\in N_{\mathcal J}(q)\bigr),
\end{equation}
and let
\[
 \Gamma_q=\{\gamma_q(\omega):\omega\in\Om_q\}
\]
be the signatures induced by admissible configurations. For
$\eta\in\Gamma_q$, define
\begin{equation}\label{eq:signature_compressed_cost}
 \bar c_q(\eta)=
 \min\{c_q(\omega):
       \omega\in\Om_q,\ \gamma_q(\omega)=\eta\}.
\end{equation}
For tree recovery, one minimizing configuration is stored for each signature.
Since admissibility is already encoded in $\Om_q$, this minimization removes
only distinctions that are invisible on every incident separator.

Introduce one variable $y_{q\eta}$ for each signature. The compressed LP has
the same normalization and separator-consistency structure as JT-LP:
\begin{subequations}\label{eq:signature_compressed_lp}
\begin{align}
 \min\quad
 &\sum_{q\in\Q}\sum_{\eta\in\Gamma_q}\bar c_q(\eta)y_{q\eta},
 \label{eq:signature_compressed_obj}\\
 \text{s.t.}\quad
 &\sum_{\eta\in\Gamma_q}y_{q\eta}=1
 &&q\in\Q,
 \label{eq:signature_compressed_norm}\\
 &\sum_{\substack{\eta\in\Gamma_q:\eta_e=\sigma}}y_{q\eta}
 -\sum_{\substack{\eta'\in\Gamma_r:\eta'_e=\sigma}}y_{r\eta'}=0
 &&e=\{q,r\}\in E,\ \sigma\in\Sigma_e,
 \label{eq:signature_compressed_sep}\\
 &y_{q\eta}\geq0
 &&q\in\Q,\ \eta\in\Gamma_q.
\end{align}
\end{subequations}
Here $\eta_e$ denotes the assignment induced by signature $\eta$ on
separator $S_e$. Configurations with the same signature have identical
coefficients in every master constraint. Their costs can therefore be
minimized within the signature class. The proposition below formalizes both
directions of this reduction: aggregation gives a feasible compressed vector,
and the stored representatives lift a compressed solution back to JT-LP.

\begin{proposition}[Exact signature compression]
\label{prop:signature_compression}
Let $P$ be the JT-LP feasible region and define
\begin{equation}\label{eq:signature_aggregation}
 y_{q\eta}
 =
 \sum_{\substack{\omega\in\Om_q:\\
                  \gamma_q(\omega)=\eta}}
 x_{q\omega}.
\end{equation}
The feasible region of \eqref{eq:signature_compressed_lp} is the image of
$P$ under \eqref{eq:signature_aggregation}. The original and compressed LPs
have the same optimal value, the compressed feasible region is integral, and
an optimal tree can be recovered from the stored minimizing configurations.
\end{proposition}

For the specified objective, compression retains a minimum-cost representative
of each signature class. Its coefficients preserve the interface, and its
stored configuration supplies tree recovery. A change of objective can change
the minimizing representative.
Under the complete split-only counting convention of
\eqref{eq:sizesplit}, with terminal predictions optimized locally, each
cluster has $F^D$ configurations. The deepest internal decision is private,
whereas the remaining $D-1$ split decisions determine the incident separator
assignments. Hence
\begin{equation}\label{eq:signature_variable_count}
 2^{D-1}F^D
 \quad\longrightarrow\quad
 2^{D-1}F^{D-1}.
\end{equation}
For the general model, early stopping and feasibility filters are already
reflected in the admissible sets $\Om_q$.

\subsection{Subtree Contraction}\label{sec:contraction}

Signature compression optimizes decisions private to one cluster. Larger
reductions become possible when adjacent clusters are grouped, because their
internal separators no longer connect the group to the rest of the model.
Subtree contraction applies the same conditional minimization to such a
group. Its boundary must still retain the ancestor decisions that determine
routing and feasibility within the private subtree.

Fix a private depth $h\in\{1,\ldots,D\}$ and let
\[
 \Q_h=\{0,1\}^{D-h}
\]
index the roots of the corresponding private subtrees. For
$t\in\Q_h$, let
\[
 \mathcal B_t=\{q\in\Q:q\text{ has prefix }t\}
\]
be the consecutive block of original clusters associated with the
depth-$h$ subtree rooted at $t$. Grouping the clusters in each
$\mathcal B_t$ produces a smaller lexicographic junction-tree chain.

A separator assignment $\sigma$ for the grouped block fixes the ancestor
decisions visible outside the private subtree and therefore fixes the
observations reaching $t$ and the inherited feasibility restrictions. Let
$\Xi_t$ be the set of such assignments that admit at least one feasible
completion. For $\sigma\in\Xi_t$, let $V_h(t,\sigma)$ denote the minimum
prediction loss and split penalties contributed by a feasible depth-$h$
subtree rooted at $t$. The cost of the corresponding contracted state is
\begin{equation}\label{eq:privatecost}
 \kappa_t(\sigma)
 =
 \sum_{\substack{v\text{ strict ancestor of }t}}
 \frac{\ell_v(\sigma)}{2^{D-h-|v|}}
 +V_h(t,\sigma).
\end{equation}
The first term is the portion of the shared-ancestor costs allocated to the
block by \eqref{eq:localcost}. Prediction losses retain the full-sample
denominator $n$. If an ancestor prediction makes the private subtree
inactive, then $V_h(t,\sigma)=0$.

The contracted master replaces each block $\mathcal B_t$ by one state for
each $\sigma\in\Xi_t$, with cost $\kappa_t(\sigma)$, one normalization
equation per block, and the separator-consistency equations induced between
adjacent blocks.

\begin{proposition}[Contraction as compression]
\label{prop:reduction_relationship}
Under the assumptions of Theorem~\ref{thm:integrality}, for every
$t\in\Q_h$ and $\sigma\in\Xi_t$,
\begin{equation}\label{eq:block_contraction_identity}
 \kappa_t(\sigma)
 =
 \min
 \left\{
 \sum_{q\in\mathcal B_t}c_q(\omega_q):
 \begin{array}{l}
 (\omega_q)_{q\in\mathcal B_t}\text{ are mutually consistent}\\
 \text{and agree with }\sigma\text{ on the external separators}
 \end{array}
 \right\},
\end{equation}
with value $+\infty$ if no feasible completion exists. Thus subtree
contraction is equivalent to grouping the block and then applying signature
compression with respect to its external separator assignments. For $h=1$,
the operation reduces to signature compression on the original clusters.
\end{proposition}

Proposition~\ref{prop:reduction_relationship} identifies the contracted cost
with the minimum cost of a compatible block in the original formulation.
The next result uses this identity to establish global optimality and tree
recovery. Replacing a private completion by its minimizer leaves every
external separator assignment unchanged, so the replacements can be made
independently across blocks.

\begin{proposition}[Exact subtree contraction]\label{prop:contraction}
Assume the ancestor-local feasibility conditions of
Section~\ref{sec:jtform}, the additive objective \eqref{eq:objective}, and
$\Tset\neq\varnothing$. The contracted master has optimal value $\OPT_D$
and an integral optimum. An optimal full classification tree can be
recovered by combining an optimal set of contracted states with the stored
minimizing private subtrees.
\end{proposition}

After contraction, another application of signature compression with the
same external separators produces no additional merging. Eliminating
separator information can, however, create new signature equivalences.
Section~\ref{sec:conditional_evaluation} gives the dynamic program used to
compute $V_h(t,\sigma)$.

\subsection{Endpoint Elimination}\label{sec:endpoint_elimination}

Compression and contraction reduce the states within each retained group.
Endpoint elimination reduces the number of groups. A leaf cluster interacts
with the rest of the junction tree through a single separator; its optimized
conditional cost can therefore be added to its neighbor. Unlike compression,
this operation also removes the corresponding consistency equations.

Let $q$ be a leaf of the compressed junction tree, let $r$ be its unique
neighbor, and write $e=\{q,r\}$. Since $q$ has only one incident separator,
its signature is identified with an assignment on $S_e$. First remove any
state of $r$ whose separator assignment has no feasible extension in $q$.
For every remaining state $\eta\in\Gamma_r$, update its cost by
\begin{equation}\label{eq:leaf_updated_cost}
 \widetilde c_r(\eta)
 =
 \bar c_r(\eta)+\bar c_q(\eta_e).
\end{equation}

\begin{proposition}[Exact endpoint elimination]
\label{prop:leaf_cluster_elimination}
After the update \eqref{eq:leaf_updated_cost}, deleting cluster $q$, its
normalization equation, and the separator equations on $e$ preserves the
optimal value and integrality. The deleted variables are recovered from
\begin{equation}\label{eq:leaf_variable_recovery}
 y_{q\sigma}
 =
 \sum_{\substack{\eta\in\Gamma_r:\eta_e=\sigma}}y_{r\eta},
 \qquad \sigma\in\Gamma_q.
\end{equation}
Combining the recovered signatures with the stored representatives from
Proposition~\ref{prop:signature_compression} recovers an optimal tree.
\end{proposition}

\begin{corollary}[Repeated elimination]\label{cor:repeated_elimination}
Any sequence of endpoint eliminations, recompressing each newly exposed
leaf with respect to its remaining separator, preserves the optimal value,
integrality, and recovery of an optimal tree. Stopping before all clusters
are eliminated leaves a smaller exact LP. Eliminating all nonroot clusters
gives the min-sum recursion developed in Section~\ref{sec:mp}.
\end{corollary}

The depth-three case illustrates the size reduction. Consider the complete
split-only model with $F$ candidate rules at each split and terminal
predictions optimized locally, before feasibility filtering.
Optimizing the deepest splits reduces the explicit variable count from
$4F^3$ to $4F^2$. Endpoint elimination then leaves $2F^2$ variables:
\[
4F^3\longrightarrow4F^2\longrightarrow2F^2.
\]

The remaining model coordinates the two sides of the tree through the root
split and has $F+1$ linearly independent equalities. Thus, for $F=100$, the
variable count falls from four million to twenty thousand. Electronic
Companion~\ref{ec:exactreductions} derives these counts.

These reductions decrease the size of the explicit LP; its coefficients are
obtained by optimizing the eliminated subtree decisions. This separates
subtree evaluation from the coordination needed to form a complete tree.
Eliminating the remaining decisions conditional on the root gives the
message-passing method in Section~\ref{sec:methods}.

The reductions above require the separator assignments to contain all
information through which an eliminated configuration or block interacts
with the rest of the model. If additional coupling is introduced, the
separator state must be augmented accordingly.

% ===== END inlined file: sections/exact_reductions =====

\section{Exact Solution Methods}\label{sec:methods}

{
JT-MP and JT-CG share the same subtree evaluator but combine its results
through message passing and column generation, respectively. Both recover
feasible trees and use lower bounds valid for the original OCT problem to
certify optimality. These bounds can incorporate results from unfinished
subtree searches.
\par}

\subsection{Conditional Subtree Evaluation}
\label{sec:acceleration}\label{sec:conditional_evaluation}

Subtree contraction removes private decisions from the explicit formulation, but their conditional costs must still be computed. At an active position $v$, let a state $s$ contain the observations reaching $v$ together with the inherited decisions needed to evaluate ancestor-local feasibility. Let $L(v,s)$ be the minimum feasible prediction loss at $v$, with value $+\infty$ when prediction is not allowed, and let $\F(v,s)$ be the admissible split features. If feature $f$ is selected, let $s_b^f$ be the child state on branch $b\in\{0,1\}$ after routing the observations and updating the inherited decisions. For a contracted block of private depth $h$, let $d$ denote the remaining depth at the current position: $d=h$ at the block root and $d=h-j$ after $j$ successive splits. Define the depth-zero value and, for $d=1,\ldots,h$, the recurrence
\begin{equation}\label{eq:privatedp}
\begin{aligned}
 V_0(v,s)&=L(v,s),\\
 V_d(v,s)&=
 \min\left\{
 L(v,s),\
 \min_{f\in\F(v,s)}
 \left[
 \lambda_{v,f}
 +V_{d-1}(v0,s_0^f)
 +V_{d-1}(v1,s_1^f)
 \right]
 \right\}.
\end{aligned}
\end{equation}
Only feasible actions are included in the minima. An inactive subtree has value zero, and prediction losses retain the full-sample denominator $n$. For a contracted block rooted at $t$, the separator assignment $\sigma$ induces an initial state $s$ in \eqref{eq:privatedp}. The notation $V_h(t,\sigma)$ from Section~\ref{sec:contraction} denotes $V_h(t,s)$ evaluated at this induced state. Minimizing actions are stored for tree recovery.

Algorithm~\ref{alg:conditional_subtree} evaluates \eqref{eq:privatedp} for any prescribed private depth $h\geq0$ by combining conditional costs from the leaves to the root and recording minimizing choices. For the algorithm, write $x=(v,s)$, $L(x)=L(v,s)$, and $x_b^f=(vb,s_b^f)$. Let $\mathcal S_j$ contain the conditional states reached from the input by $j$ successive admissible splits, for $j=0,\ldots,h$, so every state in $\mathcal S_j$ has remaining depth $d=h-j$. {States are identified by the full remaining subproblem, including inherited feasibility restrictions and position-dependent costs.} These sets specify dependencies and need not be stored simultaneously.

\begin{algorithm}[!htbp]
\caption{Exact conditional-subtree evaluation}\label{alg:conditional_subtree}
\small
\begin{tabularx}{\linewidth}{@{}r>{\raggedright\arraybackslash}X@{}}
&\textbf{Input:} Subtree root $t$, separator assignment $\sigma$ with induced state $s$, and private depth $h\geq0$.\\
&\textbf{Output:} Exact private cost $V_h(t,\sigma)$ and a minimizing subtree, or infeasibility.\\[2pt]
1&If the ancestors make $t$ inactive, return cost $0$ and an inactive subtree. Otherwise set $x_{\mathrm{in}}=(t,s)$ and $\mathcal S_0=\{x_{\mathrm{in}}\}$.\\
2&For $j=0,\ldots,h-1$, generate $\mathcal S_{j+1}$ from $x_b^f$ for $x=(u,s')\in\mathcal S_j$, $f\in\F(u,s')$, and $b\in\{0,1\}$.\\
3&For each $x\in\mathcal S_h$, set $B_0(x)\gets L(x)$; if finite, record a feasible minimizing prediction label as $A_0(x)$.\\
4&\textbf{For} $d=1,\ldots,h$ \textbf{do} \hfill\emph{Combine costs from the leaves to the root}\\
5&\quad \textbf{For each} $x=(u,s')\in\mathcal S_{h-d}$ \textbf{do}\\
6&\qquad Set $B_d(x)\gets L(x)$; if finite, record a feasible minimizing prediction label as $A_d(x)$.\\
7&\qquad \textbf{For each} $f\in\F(u,s')$ \textbf{do}\\
8&\qquad\quad Evaluate $c_f=\lambda_{u,f}+B_{d-1}(x_0^f)+B_{d-1}(x_1^f)$.\\
9&\qquad\quad If $c_f<B_d(x)$, set $B_d(x)\gets c_f$ and record $f$ and pointers to $A_{d-1}(x_0^f)$ and $A_{d-1}(x_1^f)$ as $A_d(x)$.\\
10&If $B_h(x_{\mathrm{in}})=+\infty$, return infeasibility.\\
11&Backtrack from $A_h(x_{\mathrm{in}})$: a prediction ends the branch and inactivates its descendants; a split follows both stored child choices.\\
12&Return $V_h(t,\sigma)=B_h(x_{\mathrm{in}})$ and the recovered subtree.\\
\end{tabularx}
\end{algorithm}

For $h=0$, the split-generation and cost-combination loops are empty, so the algorithm returns the best feasible prediction or infeasibility. Prediction costs minimize over all feasible classes, retain the full-sample denominator $n$, and equal $+\infty$ when prediction is infeasible. An empty split domain leaves the prediction value unchanged. Each split penalty is charged once. The allocated ancestor contribution is added to the returned private cost as in \eqref{eq:privatecost}. Successive applications of \eqref{eq:privatedp} give $B_d(x)=V_d(u,s')$ for $x=(u,s')\in\mathcal S_{h-d}$, and the recorded choices recover a minimizing feasible subtree. The algorithm assumes complete evaluation; a resource stop leaves an unresolved conditional cost rather than proving infeasibility.

The private depth $h$ controls the division of work between the explicit coordinator and the conditional evaluator. With $K=|\K|$, early stopping and repeated features give the unfiltered bound
\begin{equation}\label{eq:prefixcount}
 \sum_{t\in\Q_h}|\Xi_t|
 \leq
 2^{D-h}\left(
 F^{D-h}
 +K\sum_{j=0}^{D-h-1}F^j
 \right),
\end{equation}
where the sum is zero when $D-h=0$. Path and leaf restrictions can only reduce this count. Increasing $h$ decreases the number of explicit ancestor states but enlarges each conditional subtree problem; at $h=D$, the single conditional problem is the original OCT problem.

Table~\ref{tab:pricingsize} reports unfiltered counts for general $(D,h)$ and selected depths. Original counts assume split-only configurations with locally optimized terminal labels; representative bounds allow early stopping. For fixed $D$, increasing $h$ reduces the group count and representative bound. The baseline uses $h=1$ for $D=2,3$ and $h=3$ for $D=4,5$; Section~\ref{sec:accelerationablations} compares alternative private depths.

\begin{table}[!htbp]
\centering\small
\setlength{\tabcolsep}{4pt}
\caption{Configuration and representative-state counts.}
\label{tab:pricingsize}
\begin{tabular}{ccccc}
\toprule
$D$ & Original columns & Private depth $h$ & Groups & Representative bound\\
\midrule
$D$ & $2^{D-1}F^D$ & $h$ & $2^{D-h}$ & $2^{D-h}\bigl[F^{D-h}+K\sum_{j=0}^{D-h-1}F^j\bigr]$\\
\midrule
2 & $2F^2$ & 1 & 2 & $2(F+K)$\\
3 & $4F^3$ & 1 & 4 & $4[F^2+K(F+1)]$\\
3 & $4F^3$ & 2 & 2 & $2(F+K)$\\
\midrule
4 & $8F^4$ & 1 & 8 & $8[F^3+K(F^2+F+1)]$\\
4 & $8F^4$ & 2 & 4 & $4[F^2+K(F+1)]$\\
4 & $8F^4$ & 3 & 2 & $2(F+K)$\\
\midrule
5 & $16F^5$ & 1 & 16 & $16[F^4+K(F^3+F^2+F+1)]$\\
5 & $16F^5$ & 2 & 8 & $8[F^3+K(F^2+F+1)]$\\
5 & $16F^5$ & 3 & 4 & $4[F^2+K(F+1)]$\\
\bottomrule
\end{tabular}
\end{table}

\subsection{Message Passing}\label{sec:mp}

Complete endpoint elimination gives the standard min-sum recursion on the junction tree \citep{kschischang2001factor,wainwright2008graphical}. We first write the recursion for the original configuration formulation; after exact reduction, the same equations apply with reduced states and costs.

Root $\mathcal J$ at cluster $q_0$. For neighboring clusters $q$ and $p$, let $S_{qp}=C_q\cap C_p$, and let $N(q)$ denote the neighbors of $q$. The message from $q$ to its parent $p$ is
\begin{equation}\label{eq:exactmessages}
 M^c_{q\to p}(\sigma)=
 \min_{\substack{\omega\in\Om_q:\\
                 \omega|_{S_{qp}}=\sigma}}
 \left\{
 c_q(\omega)
 +\sum_{r\in N(q)\setminus\{p\}}
 M^c_{r\to q}(\omega|_{S_{rq}})
 \right\}.
\end{equation}
A minimum over an empty domain is $+\infty$. After an inward sweep,
\begin{equation}\label{eq:mproot}
 \OPT_D=
 \min_{\omega\in\Om_{q_0}}
 \left\{
 c_{q_0}(\omega)
 +\sum_{r\in N(q_0)}
 M^c_{r\to q_0}(\omega|_{S_{rq_0}})
 \right\}.
\end{equation}
A minimizing root state and the stored minimizing choices in \eqref{eq:exactmessages} recover an optimal tree by backtracking. Separator agreement and the running-intersection property make the recovered local states globally consistent.

For the constructed chain, once local states, costs, and separator projections are available, a sweep requires
$O(\sum_q|\Om_q|+\sum_e|\Sigma_e|)$ arithmetic operations when projections are preindexed. This is a coordination bound; constructing states and evaluating their costs are separate tasks.

Let $E(q)$ be the edges incident to cluster $q$. Orient each edge $e$ of the junction tree and let $s_{qe}\in\{+1,-1\}$ be its sign at $q$, with opposite signs at its endpoints. With $\alpha_q$ for the normalization equation and $\pi_{e\sigma}$ for separator assignment $\sigma$, the dual of JT-LP is
\begin{equation}\label{eq:rmpdual}
\begin{aligned}
 \max\quad &\sum_{q\in\Q}\alpha_q\\
 \text{s.t.}\quad
 &\alpha_q+
 \sum_{e\in E(q)}s_{qe}\pi_{e,\omega|_{S_e}}
 \leq c_q(\omega)
 &&q\in\Q,\ \omega\in\Om_q,
\end{aligned}
\end{equation}
with unrestricted dual variables.

The message recursion also constructs dual prices for separator consistency.
A message is the minimum cost of the component on one side of an edge,
conditional on the separator assignment. Charging this cost at the child
and subtracting it at the parent makes the local message inequalities match
the dual constraints. The following proposition makes this relation explicit,
including separator assignments without a feasible conditional completion.

\begin{proposition}[Min-sum messages and the JT-LP dual]
\label{prop:mpdual}
Assume $\Tset\neq\varnothing$ and the nonnegative local costs in \eqref{eq:localcost}. Root $\mathcal J$ at $q_0$ and orient every edge from child $q$ to parent $p$, with sign $+1$ at the child. Choose $H$ larger than the cost of a feasible tree. Compute finite messages $\widehat M^c_{q\to p}$ by \eqref{eq:exactmessages}, assigning value $H$ to an empty local minimization domain and using the resulting finite incoming messages in all other cases. Then
\[
 \pi_{e\sigma}=\widehat M^c_{q\to p}(\sigma),\qquad
 \alpha_{q_0}=\OPT_D,\qquad
 \alpha_q=0\quad(q\neq q_0)
\]
is optimal for \eqref{eq:rmpdual}. Together with the tree obtained by backtracking, these multipliers form a primal--dual optimal pair for JT-LP.
\end{proposition}

Thus complete elimination produces both an optimal tree and an LP optimality certificate. We call this method JT-MP. On a contracted representation, JT-MP applies the same recursion to the contracted states and their conditional costs. The cost-bound refinements used by the reported implementation are given in Electronic Companion~\ref{sec:certificate}.

\subsection{Column Generation}\label{sec:cg}

JT-CG retains only a subset of the configuration columns and exposes omitted states by pricing, in the spirit of Dantzig--Wolfe decomposition \citep{dantzig1960}.

\subsubsection{Restricted Master}\label{sec:rmp}

Let $\widehat\Om_q\subseteq\Om_q$ be the retained configurations at cluster $q$. The restricted master problem (RMP) is the restriction of JT-LP described in Corollary~\ref{cor:rmp}: it keeps the retained columns and every separator row induced by a retained column at either endpoint. A newly added column is therefore accompanied by any separator rows that it introduces, giving simultaneous column-and-row generation \citep{muter2013,spliet2024}.

Let $z_R$ denote the RMP optimum. By Corollary~\ref{cor:rmp}, every feasible RMP has an integral optimum, so $z_R$ is the value of a feasible classification tree and is an upper bound on $\OPT_D$. If an LP solver returns a fractional point on the optimal face, a tree of the same value can be recovered by the support traversal used in Theorem~\ref{thm:integrality}. The RMP is initialized from the restrictions of any feasible tree; subsequent column additions preserve feasibility.

Orient every junction-tree edge and let $s_{qe}$ be its sign at cluster $q$. Denote the dual variables for the normalization and separator rows by $\alpha_q$ and $\pi_{e\sigma}$. Multipliers for separator rows not present in the current RMP are set to zero. The RMP dual is obtained from \eqref{eq:rmpdual} by retaining only the constraints corresponding to retained columns, and strong duality gives
\[
 \sum_{q\in\Q}\alpha_q=z_R.
\]

\subsubsection{Pricing}\label{sec:pricing}

For a full-domain configuration $\omega\in\Om_q$, define its reduced cost by
\begin{equation}\label{eq:pricing}
 \widetilde c_q(\omega)
 =
 c_q(\omega)-\alpha_q
 -\sum_{e\in E(q)}
 s_{qe}\pi_{e,\omega|_{S_e}},
\end{equation}
and let
\begin{equation}\label{eq:pricingvalue}
 r_q=\min_{\omega\in\Om_q}\widetilde c_q(\omega).
\end{equation}
A negative value identifies a violated full-master dual constraint and a candidate column to add to the RMP. Because all configurations in one signature class have identical master coefficients, signature compression is compatible with pricing: only the minimum-cost representative of a signature need be considered.

After subtree contraction, the same pricing formula applies with the contracted state set and the conditional costs $\kappa_t(\sigma)$ in place of the original configurations and costs. Conditional subtree costs are independent of the RMP dual multipliers and may therefore be reused across pricing iterations.

\subsubsection{Certification}\label{sec:completealgorithm}

Corollary~\ref{cor:rmp} makes the RMP objective $z_R$ an upper bound on the
full optimum. A lower bound requires accounting for configurations omitted
from the RMP. Pricing measures the largest dual-constraint violation in each
cluster. Subtracting these violations from the corresponding normalization
prices produces a dual solution feasible for the full master, as stated next.
All pricing values must refer to the same dual vector.

\begin{proposition}[Pricing lower bound]\label{prop:pricingbound}
Let $(\alpha,\pi)$ be an optimal dual solution of a feasible RMP, with absent separator multipliers set to zero. If
$\underline r_q\leq r_q$ is a valid lower bound on the pricing value \eqref{eq:pricingvalue} computed with this same dual vector, then
\begin{equation}\label{eq:cgbound}
 \LB_{\rm CG}
 =
 z_R+\sum_{q\in\Q}\min\{0,\underline r_q\}
 \leq \OPT_D.
\end{equation}
If exact pricing gives $r_q\geq-\tau$ for every $q$, where $\tau\geq0$, then
\begin{equation}\label{eq:epsbound}
 z_R-|\Q|\tau\leq\OPT_D\leq z_R.
\end{equation}
In particular, nonnegative exact reduced costs certify optimality.
\end{proposition}

Electronic Companion~\ref{ec:cg} gives the proof. The same correction applies to any RMP-dual-feasible vector: if
$z_D=\sum_q\alpha_q$ is its dual objective, then
\begin{equation}\label{eq:cgdualgeneral}
 z_D+\sum_{q\in\Q}\min\{0,\underline r_q\}
\end{equation}
is a valid lower bound on $\OPT_D$.

{
Algorithm~\ref{alg:jtcg} gives the exact-cost JT-CG scheme in the original
configuration notation; contraction replaces these domains and costs as in
\eqref{eq:privatecost}. In the D4/D5 implementation, subtree searches provide
valid lower bounds and feasible solutions before completion. The solver uses
these bounds to select further evaluations and can certify optimality before
every subtree cost is known exactly. It retains lower bounds for all
candidate states and feasible subtrees for their upper costs. Newly priced columns are admitted after
exact cost evaluation; initial columns induced by incumbent trees may retain
upper costs. Refinement can therefore improve an existing column as well as
produce a new one. The restricted master uses the current representative
costs, and certification takes the strongest available full-domain pricing,
lower-cost message, and depth bounds. Algorithm~\ref{alg:intervalcg} in
Electronic Companion~\ref{ec:cg} specifies how these bounds are updated and
when the restricted master is solved again. A resource stop returns the best feasible tree and valid bounds.
\par}

\begin{algorithm}[!htbp]
\caption{Junction-tree column generation (JT-CG)}\label{alg:jtcg}
\small
\begin{tabularx}{\linewidth}{@{}r>{\raggedright\arraybackslash}X@{}}
&\textbf{Input:} Configuration domains $\Om_q$, feasible tree $T^0$, absolute gap tolerance $\varepsilon\geq0$, and resource limit.\\
&\textbf{Output:} A feasible tree $T$ and bounds $\LB\leq\OPT_D\leq\UB=J(T)$.\\[2pt]
1&Initialize each $\widehat\Om_q$ with the configuration induced by $T^0$ and add all induced separator rows. Set $T\gets T^0$, $\UB\gets J(T)$, and $\LB\gets-\infty$.\\
2&\textbf{while} resources remain \textbf{do}\\
3&\quad Solve the RMP; obtain $z_R$, an optimal dual vector $(\alpha,\pi)$, and a tree $T_R$ from its optimal support. Set absent separator multipliers to zero.\\
4&\quad If $J(T_R)<\UB$, set $T\gets T_R$ and $\UB\gets J(T_R)$.\\
5&\quad Price every cluster over its full domain in parallel using the same dual vector; obtain negative-reduced-cost configurations and valid bounds $\underline r_q\leq r_q$.\\
6&\quad Update $\LB\gets\max\{\LB,z_R+\sum_{q\in\Q}\min\{0,\underline r_q\}\}$. Strengthen it with the configuration-cost bound and, when available, the depth bound in Electronic Companion~\ref{ec:cg}.\\
7&\quad \textbf{if} $\UB-\LB\leq\varepsilon$ \textbf{then return} $(T,\LB,\UB)$.\\
8&\quad Add selected negative-reduced-cost configurations and every separator row they induce.\\
9&\quad If no column was found and pricing is incomplete, retain the current dual vector and resume unfinished pricing at Step~5 without resolving the RMP, until a column, certificate, or resource stop is obtained.\\
10&\quad If pricing is exact and $r_q\geq0$ for every $q\in\Q$, return $(T,\UB,\UB)$.\\
11&\textbf{end while}; \textbf{return} $(T,\LB,\UB)$.\\
\end{tabularx}
\end{algorithm}

The pseudocode uses exact reduced-cost signs and no column deletion. Each pricing correction uses bounds from one fixed dual solution. If the resource limit interrupts a solve or pricing call, return the incumbent and the last valid bounds. With complete exact pricing and no resource limit, every nonterminal iteration adds a previously absent configuration from a finite domain, so the procedure terminates finitely.

\subsection{Parallel Evaluation}\label{sec:parallel}

For a fixed separator assignment, the conditional subtree problem in Section~\ref{sec:conditional_evaluation} depends only on the observations and feasibility information carried by that state. Distinct conditional requests can therefore be evaluated independently. Within one request, candidate split and prediction costs can also be evaluated in batches before being combined through \eqref{eq:privatedp}.

Parallel workers evaluate conditional costs. JT-MP combines the resulting messages, while JT-CG prices states under a common current dual vector; the coordinator selects global states and checks certification.

{
A cached conditional value is indexed by the routed observations, remaining
depth, inherited feasibility restrictions, and all costs and restrictions
of the remaining subtree. Position is part of this identity when it changes
the subproblem. The computational study uses uniform observation weights and
a common split penalty. Separator states remain distinct unless their master
coefficients agree. Electronic Companion~\ref{ec:costbounds} gives the bound
reuse conditions; the experiments distinguish fixed-task throughput from
end-to-end solver runtime.
\par}

{
For parallel pricing, partition $\Om_q$ into disjoint blocks $\Om_{qh}$, $h=1,\ldots,H_q$. Then
\begin{equation}\label{eq:partitionpricing}
 r_q=\min_{1\leq h\leq H_q}\min_{\omega\in\Om_{qh}}\widetilde c_q(\omega).
\end{equation}
Valid block bounds $\underline r_{qh}$ combine as
\begin{equation}\label{eq:ecparallelbound}
 \underline r_q=\min_{1\leq h\leq H_q}\underline r_{qh}\leq r_q.
\end{equation}
All blocks, including unfinished ones, contribute valid bounds under the same dual vector to \eqref{eq:ecparallelbound}; completed blocks may return columns immediately.
\par}

% ===== END inlined file: sections/algorithm_and_acceleration =====

% ===== BEGIN inlined file: sections/experiments =====

\section{Computational Experiments}\label{sec:experiments}

The computational study has three parts. We first compare the proposed
methods with existing exact OCT solvers. We then examine the effect of the
exact reductions on formulation size and end-to-end solution runtime. Finally,
we evaluate the computational impact of subtree contraction and parallel
conditional-subtree evaluation.

\subsection{Experimental Design}\label{sec:setup}

We use the 11 public classification datasets summarized in Table~\ref{tab:datasets}, which reports the number of observations, binary features, and classes after preprocessing. Numerical variables are converted to cumulative empirical-decile predicates, with repeated thresholds removed, and categorical variables to equality indicators. All methods receive the same resulting binary matrices and use all observations for training. Maximum depths $D\in\{2,3,4,5\}$ and split penalties $\lambda\in\{0,0.01\}$ give 88 benchmark settings.

\begin{table}[!htbp]
\centering\small
\caption{Benchmark datasets.}\label{tab:datasets}
\begin{tabular}{lrrr}
\toprule
Dataset & Observations $n$ & Binary features $F$ & Classes $|\K|$\\
\midrule
\texttt{avila} & 20,867 & 85 & 12\\
\texttt{banknote} & 1,372 & 36 & 2\\
\texttt{compas} & 12,381 & 71 & 2\\
\texttt{diabetic} & 101,766 & 315 & 3\\
\texttt{fico} & 10,459 & 159 & 2\\
\texttt{give} & 150,000 & 57 & 2\\
\texttt{htru2} & 17,898 & 72 & 2\\
\texttt{letter} & 20,000 & 99 & 26\\
\texttt{skin} & 245,057 & 27 & 2\\
\texttt{spambase} & 4,601 & 152 & 2\\
\texttt{transactions} & 786,363 & 131 & 2\\
\bottomrule
\end{tabular}
\end{table}

Runs use an Intel Core i9-14900KF central processing unit (CPU), 64~GiB of host memory, and an NVIDIA
RTX~4080 SUPER graphics processing unit (GPU). The nominal time limit is 600 seconds. Gurobi~13.0.3 is used to solve the LP models and the mixed-integer programming (MIP) models.

JT runs use an absolute gap tolerance of $10^{-7}$ for numerical optimality
certification. The lower bounds come from the methods' full-domain
certificates; upper bounds are attained by recovered trees. Each returned
tree is independently rerouted and reevaluated under \eqref{eq:objective}
to check its feasibility and objective.
Our code and datasets are available in the GitHub repository (\url{https://github.com/Tommytutu/JT-OCT}).

We compare JT-LP, JT-CG and JT-MP with the mixed-integer method BendersOCT (BOCT)
\citep{aghaei2024strong} and the exact-search methods STreeD
\citep{vanderlinden2023}, MurTree \citep{demirovic2022murtree}, GOSDT
\citep{lin2020gosdt}, Branches \citep{chaouki2025}, and DL8.5
\citep{aglin2020dl85}. We additionally compare with RolloTree's OCT-2 module \citep{organ2026rolling} only at $D=2$ and $\lambda=0$, under its complete-tree domain. We also report heuristic column generation (HCG) \citep{patel2024cg}, using screened feature sets, and CART
\citep{breiman1984cart} as a feasible-tree baseline. {GOSDT results use \texttt{gosdt-guesses} \citep{mctavish2022fast}, version~1.0.4. For a binary tree, the leaf count is one plus the split count, so we subtract the constant $\lambda$ from its leaf-penalized objective and bounds to match \eqref{eq:objective}.} Among the exact methods,
The BendersOCT
objective is modified to match the common training-loss and split-penalty
objective. {DL8.5 uses a maximum-depth constraint and permits smaller trees; its experiments here cover $\lambda=0$.}
All external implementations are run on the same machine using their
documented interfaces. 

{
The JT implementations allow early stopping and prohibit repeated features on
a path. JT-LP solves the formulation after signature compression and endpoint
elimination, using direct minimization when one cluster remains. At D2/D3,
JT-CG uses a specialized exact shallow procedure. At D4/D5, it evaluates
subtrees of private depth $h=3$ and uses bounds from unfinished searches as
specified in Algorithm~\ref{alg:intervalcg}. Of the 44 settings, 42 enter
the restricted-master loop. For \texttt{transactions} at both depths with
$\lambda=0.01$, the depth bound in \eqref{eq:ecdepthbound} closes the gap
before this loop. JT-MP uses complete min-sum elimination.
The automatic conditional-evaluation policy selects CPU or GPU execution;
for D4/D5 JT-CG, it selects CPU for the four \texttt{banknote} settings and
GPU for the remaining 40. These D4/D5 runs use up to eight CPU workers;
the shallow procedures retain their recorded thread settings.
\par}

\subsection{Comparison with Existing Exact Methods}\label{sec:performance}

{
The comparison uses the common objective \eqref{eq:objective} and the
method-specific domains described above. We report the shallow and deeper
results separately to distinguish the specialized D2/D3 procedure from the
D4/D5 restricted-master implementation.
\par}

Table~\ref{tab:depthbenchmark} reports returned objective values, runtimes, and optimality certificates.
At depth two, the JT formulation specializes to the depth-two configuration
structure of \citet{organ2026rolling} under matched modeling assumptions.
RolloTree therefore provides a direct reference for the shallow formulation.
For $\lambda=0$, both RolloTree and JT-LP certify all 11 datasets, while their
arithmetic-mean runtimes are 40.40 and 0.09 seconds, respectively. The
JT-LP implementation reduces the explicit formulation before optimization
and uses parallel computation in local-cost construction.

BOCT certifies 2 of the 88 settings and reaches the time limit on most runs.
Its numbers of variables and constraints increase with sample size, making
the large datasets in this benchmark particularly demanding. For fixed
depth, features, and classes, the size bounds of the JT formulation do not
depend on the number of observations. This structural difference helps
explain the runtime advantage of the JT methods.

\begin{table}[!htbp]
\centering\scriptsize
\caption{Optimization results by depth and split penalty. ObjVal and AverageRuntime are arithmetic means; OPT is the number of method-supplied certificates among 11 datasets.}\label{tab:depthbenchmark}
\setlength{\tabcolsep}{3.2pt}\renewcommand{\arraystretch}{0.96}
\resizebox{\linewidth}{!}{%
\begin{tabular}{lrrrrrrrrrrrr}
\toprule
& \multicolumn{4}{c}{ObjVal} & \multicolumn{4}{c}{AverageRuntime (s)} & \multicolumn{4}{c}{OPT}\\
\cmidrule(lr){2-5}\cmidrule(lr){6-9}\cmidrule(lr){10-13}
Model & $D=2$ & $D=3$ & $D=4$ & $D=5$ & $D=2$ & $D=3$ & $D=4$ & $D=5$ & $D=2$ & $D=3$ & $D=4$ & $D=5$\\
\midrule
\multicolumn{13}{l}{\textit{Panel A: $\lambda=0$}}\\
CART & 0.2597 & 0.2345 & 0.2135 & 0.1928 & --- & --- & --- & --- & --- & --- & --- & ---\\
RolloTree & 0.2466 & --- & --- & --- & 40.40 & --- & --- & --- & 11 & --- & --- & ---\\
BOCT & 0.2552 & 0.2322 & 0.2130 & 0.1928 & 558.30 & 600.93 & 601.52 & 602.07 & 1 & 0 & 0 & 0\\
HCG & 0.2566 & 0.2332 & 0.2114 & 0.1927 & 79.29 & 111.86 & 185.27 & 282.58 & --- & --- & --- & ---\\
DL8.5 & 0.2466 & 0.2188 & 0.2055 & 0.2262 & 2.03 & 126.82 & 412.71 & 503.39 & 11 & 9 & 4 & 2\\
GOSDT & 0.2466 & 0.2184$_{9/11}$ & 0.0865$_{5/11}$ & 0.0057$_{2/11}$ & 3.33 & 45.61 & 132.89 & 73.38 & 11 & 9 & 4 & 2\\
MurTree & 0.2466 & 0.2188 & 0.1894$_{9/11}$ & 0.0702$_{4/11}$ & 2.20 & 12.49 & 160.95 & 437.76 & 11 & 11 & 9 & 4\\
STreeD & 0.2466 & 0.2188 & 0.1947 & 0.1788 & 2.20 & 4.80 & 133.83 & 381.09 & 11 & 11 & 10 & 5\\
Branches & 0.2466 & 0.2192 & 0.2194 & 0.2352 & 8.55 & 42.78 & 80.14 & 79.10 & 11 & 8 & 2 & 1\\
JT-LP & 0.2466 & 0.2188 & 0.1947 & 0.1712 & 0.09 & 1.52 & 27.63 & 201.90 & 11 & 11 & 11 & 9\\
JT-CG & 0.2466 & 0.2188 & 0.1947 & 0.1706 & \textbf{0.07} & \textbf{0.40} & \textbf{20.59} & 160.96 & 11 & 11 & 11 & 9\\
JT-MP & 0.2466 & 0.2188 & 0.1947 & 0.1705 & 0.23 & 0.55 & 23.52 & 174.92 & 11 & 11 & 11 & 9\\
\midrule
\multicolumn{13}{l}{\textit{Panel B: $\lambda=0.01$}}\\
CART & 0.2897 & 0.3027 & 0.3517 & 0.4555 & --- & --- & --- & --- & --- & --- & --- & ---\\
BOCT & 0.2833 & 0.2863 & 0.3113 & 0.3700 & 548.67 & 601.19 & 601.29 & 601.48 & 1 & 0 & 0 & 0\\
GOSDT & 0.2674 & 0.2548 & 0.2366$_{9/11}$ & 0.0579$_{5/11}$ & 2.19 & 79.36 & 220.71 & 72.56 & 11 & 10 & 6 & 5\\
MurTree & 0.2674 & 0.2548 & 0.2279$_{10/11}$ & 0.0974$_{6/11}$ & 3.51 & 4.21 & 107.82 & 355.55 & 11 & 11 & 10 & 6\\
STreeD & 0.2674 & 0.2548 & 0.2476 & 0.2478 & 1.95 & 3.13 & 87.20 & 283.38 & 11 & 11 & 11 & 7\\
Branches & 0.2674 & 0.2548 & 0.2548 & 0.2548 & 10.58 & 67.27 & 137.89 & 146.69 & 11 & 9 & 4 & 3\\
JT-LP & 0.2674 & 0.2548 & 0.2476 & 0.2436 & 0.09 & 1.46 & 27.23 & 201.05 & 11 & 11 & 11 & 9\\
JT-CG & 0.2674 & 0.2548 & 0.2476 & 0.2436 & \textbf{0.07} & \textbf{0.30} & \textbf{10.00} & 88.09 & 11 & 11 & 11 & 10\\
JT-MP & 0.2674 & 0.2548 & 0.2476 & 0.2436 & 0.23 & 0.48 & 10.98 & 92.69 & 11 & 11 & 11 & 10\\
\bottomrule
\end{tabular}%
}
\begin{minipage}{\linewidth}\footnotesize\vspace{2pt}
ObjVal averages returned feasible objectives, including CART fallback trees for BOCT runs without a solver incumbent; a subscript $n/11$ gives the number of available objectives. AverageRuntime averages recorded runtimes over attempted runs, including resource-limited runs. OPT counts each method's certificates for its stated domain. {RolloTree uses a complete-tree domain at $D=2$; DL8.5 uses a maximum-depth domain. Both are reported only at $\lambda=0$.} HCG uses screened features and has no full-domain certificate. The JT results compare implementations under their recorded CPU/GPU policies.
\end{minipage}
\end{table}

{
Among the external exact methods, STreeD certifies the largest number of
settings, 77 of 88. JT-LP certifies 84, while JT-CG and JT-MP each certify
85. On the 77 settings jointly certified by JT-CG and STreeD, JT-CG is
faster on 76, with a paired geometric-mean runtime ratio
$t_{\mathrm{STreeD}}/t_{\mathrm{JT\text{-}CG}}$ of 21.79.
For D2/D3, all 44 settings are jointly certified, with JT-CG faster on all
44 and a geometric-mean ratio of 38.20. For D4/D5, the corresponding counts
are 33 jointly certified settings and 32 faster cases, with a ratio of
10.31. Excluding the two cases certified before the restricted-master loop
gives 9.38 on 31 jointly certified settings. 
\par}

The runtime difference decreases with depth. Using the unrounded
arithmetic-mean runtimes underlying Table~\ref{tab:depthbenchmark}, the STreeD/JT-CG ratio is 28.42--31.35 at
depth two, 10.27--12.03 at depth three, 6.50--8.72 at depth four, and
2.37--3.22 at depth five, where each range corresponds to the two values of
$\lambda$. Thus, the largest runtime differences occur at depths two and
three. At depth five, the difference is more apparent in certification
coverage: JT-CG certifies 19 of the 22 settings, compared with 12 for STreeD.

{
Among the JT implementations, JT-CG has the lowest arithmetic-mean runtime in every depth--penalty group in Table~\ref{tab:depthbenchmark}.
At D4/D5, the ratios of JT-MP to JT-CG mean runtimes range from 1.05 to 1.14,
with identical certification counts. This difference is smaller than the
corresponding comparison with STreeD. JT-LP is faster than JT-MP at depth two,
whereas JT-MP has the
lower mean runtime from depth three onward.
\par}

The effect of the split penalty also differs across the three JT
implementations. The runtime of JT-LP changes little between $\lambda=0$
and $\lambda=0.01$: changing the penalty changes the configuration costs but
not the size of the reduced LP. The effect is larger for JT-CG and JT-MP. At
depth five, the mean runtime decreases from 160.96 to 88.09 seconds for
JT-CG and from 174.92 to 92.69 seconds for JT-MP when $\lambda$ increases from
0 to 0.01. A positive split penalty enters the conditional lower bounds and can
enable pruning before full subtree evaluation. STreeD shows the same direction of change in mean
runtime at depths three through five.

Figure~\ref{fig:benchmarkdistributions} provides the corresponding
distributional comparison. Panel~(a) reports the cumulative number of settings
certified as a function of runtime. Panel~(b) reports the distribution of
$t_{\mathrm{STreeD}}/t_{\mathrm{JT}}$ over settings certified by both methods.
The first panel therefore shows both the rate at which certificates are
obtained and the final certification coverage, while the second shows how the
runtime differences are distributed across individual benchmark settings.

\begin{figure}[!htbp]
\centering
\includegraphics[width=\linewidth]{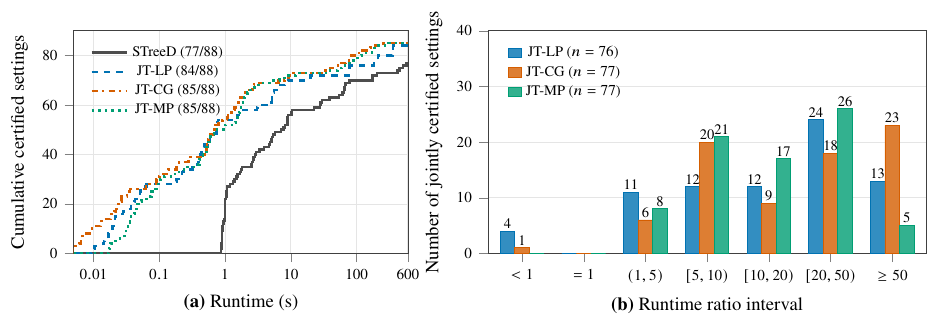}
\caption{Certification runtime profiles and paired runtime ratios relative to STreeD.}
\label{fig:benchmarkdistributions}
\end{figure}

The three settings not certified by JT-CG within the time limit are all at
depth five: \texttt{diabetic} for both values of $\lambda$ and
\texttt{transactions} for $\lambda=0$. Their final lower and upper bounds,
together with the complete per-instance objectives, gaps, runtimes, and
termination statuses, are reported in Electronic
Companion~\ref{app:results}.

\subsection{Effect of Structural Reduction}
\label{sec:structural-reduction}\label{sec:ablations}\label{sec:progressivereduction}

{
This experiment asks how much signature compression (SC) and endpoint
elimination (EE) reduce the explicit formulation, and whether the smaller
coordination system reduces coordination and end-to-end runtime. We use
the depth-three complete-tree model with $\lambda=0$, seven splits,
eight nonempty leaves, no repeated feature on a path, and locally optimized
terminal labels. LP-base is unreduced; LP+SC retains one minimum-cost
representative per separator signature; LP+SC+EE also eliminates the endpoint
clusters; and message passing (MP) eliminates the remaining chain. The three
reduced variants share data, conditional costs, execution policy, and timing
boundary. JT-LP in Section~\ref{sec:performance} is LP+SC+EE.
\par}

\begin{table}[!htbp]
\centering\scriptsize
\caption{Effect of structural reduction in the depth-three complete-tree experiment with $\lambda=0$.}
\label{tab:structural-reduction}\label{tab:structurald3pending}
\setlength{\tabcolsep}{3.2pt}
\resizebox{\linewidth}{!}{%
\begin{tabular}{lrrrrrrrrrrrr}
\toprule
& \multicolumn{2}{c}{Data} & \multicolumn{3}{c}{Number of columns} & \multicolumn{3}{c}{Number of rows} & \multicolumn{4}{c}{Runtime (s)}\\
\cmidrule(lr){2-3}\cmidrule(lr){4-6}\cmidrule(lr){7-9}\cmidrule(lr){10-13}
Dataset & $n$ & $F$ & LP-base & LP+SC & LP+SC+EE & LP-base & LP+SC & LP+SC+EE & LP-base & LP+SC & LP+SC+EE & MP\\
\midrule
avila & 20,867 & 85 & 2,101,044 & 27,898 & 13,618 & 14,369 & 14,369 & 87 & 36.069 & 0.484 & 0.444 & \textbf{0.424}\\
banknote & 1,372 & 36 & 107,436 & 4,594 & 2,074 & 2,560 & 2,560 & 38 & 0.186 & 0.012 & 0.008 & \textbf{0.004}\\
compas & 12,381 & 71 & 1,138,266 & 19,172 & 9,232 & 10,015 & 10,015 & 73 & 13.109 & 0.061 & 0.043 & \textbf{0.032}\\
diabetic & 101,766 & 315 & 71,564,408 & 335,582 & 143,013 & 192,888 & 192,888 & 317 & M & 6.497 & 6.038 & \textbf{5.962}\\
fico & 10,459 & 159 & 14,230,964 & 97,968 & 47,882 & 50,249 & 50,249 & 161 & T & 0.640 & 0.550 & \textbf{0.512}\\
give & 150,000 & 57 & 592,276 & 12,350 & 5,966 & 6,445 & 6,445 & 59 & 5.819 & 0.184 & 0.178 & \textbf{0.158}\\
htru2 & 17,898 & 72 & 1,023,958 & 19,098 & 8,874 & 10,300 & 10,300 & 74 & 10.773 & 0.065 & 0.050 & \textbf{0.044}\\
letter & 20,000 & 99 & 3,389,466 & 38,094 & 18,690 & 19,507 & 19,507 & 101 & 61.871 & 0.798 & 0.746 & \textbf{0.714}\\
skin & 245,057 & 27 & 45,882 & 2,592 & 1,188 & 1,435 & 1,435 & 29 & 0.266 & 0.055 & 0.054 & \textbf{0.051}\\
spambase & 4,601 & 152 & 12,833,790 & 90,832 & 44,928 & 46,060 & 46,060 & 154 & 333.997 & 0.587 & 0.483 & \textbf{0.439}\\
transactions & 786,363 & 131 & 7,870,986 & 66,674 & 32,614 & 34,195 & 34,195 & 133 & T & \textbf{5.133} & 5.217 & 5.212\\
\bottomrule
\end{tabular}%
}
\begin{minipage}{0.98\linewidth}
\footnotesize\vspace{2pt}
\textit{Notes.} All reduced-model runs pass the independent tree and objective
audit. M denotes a memory limit and T a time limit. LP+SC, LP+SC+EE, and MP
use the same matched evaluation policy. LP-base is retained as an unreduced
size and tractability reference.
\end{minipage}
\end{table}

{
SC reduces the total number of columns from 114,898,476 to 714,854, a factor
of 160.7. EE reduces it further to 328,079 columns and reduces submitted
equalities from 388,023 to 1,226. The reductions are especially pronounced
for \texttt{diabetic}: 71.6 million columns in LP-base, 335,582 after SC, and
143,013 after SC+EE. LP-base reaches a memory limit on \texttt{diabetic} and
time limits on \texttt{fico} and \texttt{transactions}; all reduced variants
complete the 11 datasets.

Phase runtimes separate coordination from cost construction. Across the 11
datasets, EE reduces coordination runtime by a geometric-mean factor of 2.57,
whereas cost-preparation runtime changes by a factor of 1.02 and total runtime by
1.16. A fixed-cost control, which supplies complete representative-cost tables
to each coordinator, gives the same conclusion: on the 19 available D5 tables,
EE reduces coordination runtime by factors of 2.42 for LP and 4.85 for eager CG.
Thus SC produces most of the reduction in the explicit domain, while EE acts
primarily on coordination; after cost evaluation becomes dominant, the latter
has a smaller effect on total runtime.
\par}

\subsection{Effect of Subtree Contraction}
\label{sec:subtree-contraction}\label{sec:contraction_experiments}\label{sec:accelerationablations}

{
This experiment asks how contraction and private depth change the amount of
conditional work. We compare MP with and without contraction
on the 44 D4/D5 settings under CPU8 and a 600-second limit. The contracted
variant uses private depth $h=3$; both variants use the same data and tree
domain. Contraction raises certification from 26 to 29 settings. On the 26
joint certificates, it is faster in every pair, reduces arithmetic mean runtime
from 133.56 to 23.64 seconds, and gives a geometric-mean speedup of 4.30.
\par}

\begin{figure}[!htbp]
\centering
\includegraphics[width=0.86\linewidth]{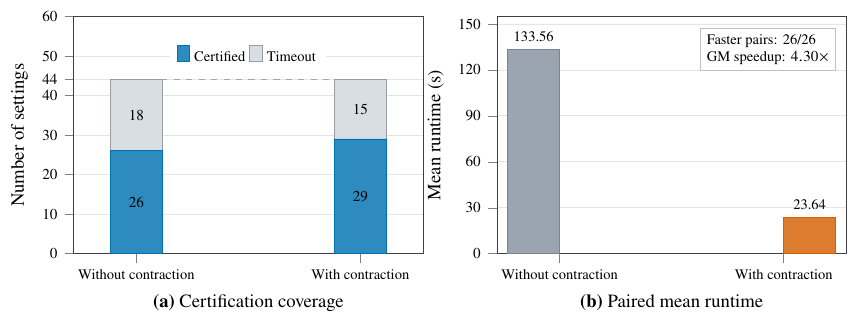}
\caption{{Subtree contraction under CPU8 and a 600-second
limit. Panel (a) reports certification coverage over 44 D4/D5 settings; panel
(b) reports arithmetic mean runtime on the 26 jointly certified settings.}}
\label{fig:contraction-summary}
\end{figure}

{
We compare private-subtree depths $h=2$ and $h=3$ using the same
implementation under CPU8. Table~\ref{tab:privatedepth} reports certification over all 44
settings and work statistics on the 20 certified at both depths. Increasing $h$ from two to three moves work from the explicit
chain into larger private problems, reducing retained states and oracle calls
enough to lower total runtime. Oracle calls at the two depths solve different
private subproblems and therefore measure invocation counts, rather than
equal-sized units of work.
\par}

\begin{table}[!htbp]
\centering\small
\caption{{Effect of private-subtree depth on 44 D4/D5 settings under
CPU8. State and call counts are totals; runtimes are arithmetic means over
the 20 jointly certified settings. At $h=2$, 16 runs reach the 200,000-state
capacity and eight reach the time limit; at $h=3$, 14 reach the time limit.}}
\label{tab:privatedepth}
{
\setlength{\tabcolsep}{4pt}
\begin{tabular}{lrrrrrr}
\toprule
\shortstack[l]{Private\\depth} & \shortstack{Certified\\settings} & \shortstack{Retained\\states} & \shortstack{Private-subtree\\oracle calls} & \shortstack{Mean\\runtime (s)} & \shortstack{Faster\\pairs} & \shortstack{GM speedup\\relative to $h=2$}\\
\midrule
$h=2$ & 20/44 & 983,728 & 295,636 & 31.78 & 2/20 & 1.00\\
$h=3$ & 30/44 & 9,304 & 4,337 & 21.22 & 18/20 & 3.28\\
\bottomrule
\end{tabular}}
\end{table}

\subsection{Parallel Conditional Evaluation}
\label{sec:parallel-evaluation}\label{sec:parallelablations}

{
We first measure throughput on identical conditional-subtree task lists.
CPU1, CPU8, and GPU complete all three repetitions on
38, 41, and 44 of the 44 D4/D5 settings, respectively. On the 38 settings
completed by all three modes, the geometric-mean speedups over CPU1 are 4.41
for CPU8 and 9.81 for GPU. This fixed-task experiment isolates conditional
evaluation from the state-selection decisions of the complete solvers.
\par}

\begin{figure}[H]
    \centering
    \includegraphics[width=0.82\linewidth]{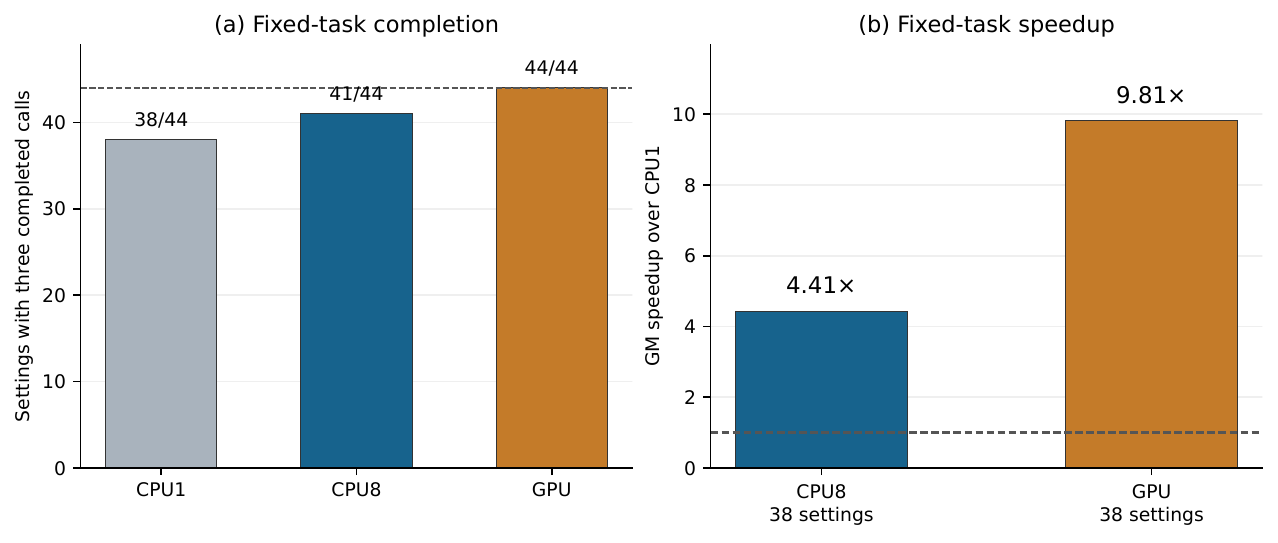}
    \caption{{Parallel evaluation of identical conditional-subtree
    tasks. Panel (a) reports settings for which all three repetitions finish
    within 600 seconds. Panel (b) reports geometric-mean speedups over CPU1 on
    the 38 settings completed by all three modes, using median runtime.}}
    \label{fig:parallel-evaluation}\label{fig:parallelcoverage}
\end{figure}

{
We next compare CG and MP in matched complete solvers,
sharing the evaluator, cache rules, initial feasible tree, hardware policy,
and 600-second limit. Each method certifies 31 of 44 settings under CPU8 and
41 under Auto. Table~\ref{tab:matched-coordination} compares conditional work
on the jointly certified settings that enter both coordination procedures.
CG makes 4.8\% fewer exact oracle calls under CPU8 and 7.8\% fewer under Auto.
Evaluation dominates runtime for both methods. The saved evaluations
are partly offset by CG's master solves: on these paired settings, the
geometric-mean ratios $t_{\mathrm{MP}}/t_{\mathrm{CG}}$ are 0.921 and 0.957.
\par}

\begin{table}[!htbp]
\centering\small
\caption{{Matched conditional work and runtime. Statistics use
29 jointly certified CPU8 settings and 39 jointly certified Auto settings
that enter both coordination procedures. Calls are totals; runtimes are
arithmetic means in seconds.}}
\label{tab:matched-coordination}
{
\begin{tabular}{llrrrrr}
\toprule
Policy & Method & Exact oracle calls & Evaluation & RMP & Messages & Total\\
\midrule
CPU8 & CG & 60,053 & 42.86 & 1.435 & 0.039 & 46.08\\
 & MP & 63,101 & 43.18 & 0.000 & 0.026 & 44.74\\
Auto & CG & 210,106 & 29.53 & 1.029 & 0.025 & 32.25\\
 & MP & 227,914 & 32.72 & 0.000 & 0.030 & 34.49\\
\bottomrule
\end{tabular}}
\end{table}

{
Finally, we compare CPU8 and Auto in the complete solvers. On the 31
settings certified under both policies, Auto gives
geometric-mean speedups over CPU8 of 3.71 for CG and 3.80 for MP. Auto includes
the tested CPU/GPU selection and batching policy; the fixed-task experiment
above provides the direct throughput comparison on an identical workload.
\par}

\FloatBarrier
% ===== END inlined file: sections/experiments =====

\FloatBarrier
% ===== BEGIN inlined file: sections/conclusions =====
\section{Conclusion}\label{sec:conclusions}

{
This paper develops an exact linear programming approach to optimal
classification trees of any prescribed maximum depth with binary features.
The model accommodates early stopping and minimum leaf support, and an
optimal tree can be recovered from an optimal LP solution. It describes
tree decisions without introducing variables and constraints for individual
training observations; the data determine the costs and feasibility of
these decisions. Exact reductions preserve the optimal objective while
shrinking the model and separating conditional-subtree optimization from
the task of combining compatible decisions. This structure supports two
exact solution methods, JT-CG and JT-MP, and parallel evaluation of the
conditional-subtree problems.

The experiments show that the reductions substantially decrease the explicit
model size and that parallel subtree evaluation accelerates the resulting
solution methods. Across 88 settings from 11 datasets, four depths, and two
split penalties, JT-CG and JT-MP each certify 85 settings, compared with 77
for STreeD. On the 77 settings certified by both JT-CG and STreeD, JT-CG
achieves a geometric-mean speedup of 21.79. The matched comparisons show
that JT-CG solves fewer conditional-subtree problems, while JT-MP combines
their results faster, leading to similar overall runtimes.

The remaining depth-five settings highlight the difficulty of proving
optimality after a good tree has been found. Stronger bounds on conditional
subtree costs and more selective allocation of search effort are therefore
promising directions for solving deeper OCT problems. The formulation also
provides a basis for studying other additive prediction losses and multiway
classification trees.
\par}
% ===== END inlined file: sections/conclusions =====

% ===== BEGIN inlined file: sections/data_availability =====
\section*{Data and Code Availability}
The Electronic Companion contains proofs, certificate details, and supporting computational results. The anonymized reproducibility artifact supplied with the submission contains preprocessing scripts, exact commands, software versions, raw logs, and per-instance outputs.
% ===== END inlined file: sections/data_availability =====

\begingroup
\bibliographystyle{plainnat}
\bibliography{references}
\endgroup
% ===== END inlined file: body =====

\clearpage
\ECSwitch
\renewcommand{\theHsection}{EC.\arabic{section}}
\renewcommand{\theHsubsection}{EC.\arabic{section}.\arabic{subsection}}
\renewcommand{\theHsubsubsection}{EC.\arabic{section}.\arabic{subsection}.\arabic{subsubsection}}
\renewcommand{\theHequation}{EC.\arabic{equation}}
\renewcommand{\theHfigure}{EC.\arabic{figure}}
\renewcommand{\theHtable}{EC.\arabic{table}}
\renewcommand{\theHtheorem}{EC.\arabic{theorem}}
\renewcommand{\theHlemma}{EC.\arabic{lemma}}
\renewcommand{\theHproposition}{EC.\arabic{proposition}}
\renewcommand{\theHcorollary}{EC.\arabic{corollary}}
\renewcommand{\theHdefinition}{EC.\arabic{definition}}
\renewcommand{\theHremark}{EC.\arabic{remark}}
\ECHead{Electronic Companion}
% ===== BEGIN inlined file: appendices =====
% ===== BEGIN inlined file: sections/appendix_theory =====
\section{Proofs for the Junction-Tree Formulation}\label{ec:hull}

\subsection{Validity of the cluster chain}\label{ec:chainproofsection}
\begin{proof}{Proof of Proposition~\ref{prop:jtconstruction}.}
Fix an internal position $v$. A cluster $C_q$ contains $v$ exactly when $v$ lies on the path from the root to $q$, including both endpoints. These indices occur consecutively in lexicographic order, so the clusters containing $v$ form a connected part of $\mathcal J$. A terminal position belongs to only one cluster. The running-intersection condition in Definition~\ref{def:junctiontree} therefore holds.
\end{proof}

\subsection{Convex-hull property and tree recovery}
\begin{proof}{Proof of Proposition~\ref{prop:costdecomp}.}
Fix $T\in\Tset$. For every split node $v\in B(T)$, the term $\lambda_{v,f_v}/m_v$ occurs in each of the $m_v$ clusters containing $v$. If $v\in L(T)$ is a prediction leaf, then $\ell_v(T|_{C_q})$ is the same for every cluster containing $v$, because each such cluster contains the complete ancestor path of $v$. Consequently,
\begin{align*}
 \sum_{q\in\Q}c_q(T|_{C_q})
 &=\sum_{v\in B(T)}\sum_{q:v\in C_q}\frac{\lambda_{v,f_v}}{m_v}
   +\sum_{v\in L(T)}\sum_{q:v\in C_q}\frac{\ell_v(T|_{C_q})}{m_v}\\
 &=\sum_{v\in B(T)}\lambda_{v,f_v}
   +\frac{1}{n}\sum_{i\in I}\mathbf 1\{T(z_i)\neq y_i\}
 =J(T).
\end{align*}
\end{proof}

\begin{proof}{Proof of Theorem~\ref{thm:integrality}.}

\emph{OCT configuration--tree correspondence.}
Restrict a feasible tree $T\in\Tset$ to each cluster. Every structural, path, and individual-leaf condition has its full scope in a cluster, so these restrictions are admissible and agree on separators.

Conversely, choose separator-consistent admissible configurations. Running intersection makes the action at each position unique. Every parent--child pair is covered, so the assembled actions enforce an active root, active children exactly when the parent splits, inactive descendants after a prediction, and no terminal split. Each prediction position occurs with its complete ancestor path; its routed observations and local feasibility checks therefore coincide with those of the assembled tree. Restriction and assembly are inverse maps, giving the claimed bijection and
\begin{equation}\label{eq:ecfirstinclusion}
 \conv\{x^T:T\in\Tset\}\subseteq P.
\end{equation}

\emph{Junction-tree convex decomposition.}
{
The standard junction-tree decomposition \citep[Proposition~2.1]{wainwright2008graphical}
and configuration-LP integrality argument \citep[Lemma~2.1 and Theorem~1.1]{kolman2015}
give a constructive recovery procedure.
\par}

Root the junction tree and select a positive-weight configuration of $x\in P$ there. Each separator equation supplies a positive-weight child configuration agreeing with the selected parent. Continuing gives a consistent selection and hence a feasible tree $T$, all of whose selected coordinates have positive weight.

Let $\delta$ be the smallest weight in $x$ among the configurations selected for $T$. Then $0<\delta\leq1$. If $\delta=1$, normalization gives $x=x^T$. Otherwise, $(x-\delta x^T)/(1-\delta)\in P$.
The numerator is nonnegative, has cluster sums $1-\delta$, and satisfies the separator equations. Each step removes a positive coordinate without introducing any, so iteration terminates and yields
\begin{equation}\label{eq:ecdecomposition}
 x=\sum_{T\in\Tset}p_T x^T,
 \qquad p_T\geq0,\qquad\sum_{T\in\Tset}p_T=1.
\end{equation}
Together with \eqref{eq:ecfirstinclusion}, this yields \eqref{eq:hull}.

\emph{Objective value and tree recovery.}
Proposition~\ref{prop:costdecomp} makes the objective at $x$ equal to $\sum_T p_T J(T)$. Hence the LP and the IP have value $\OPT_D$. If $x$ is optimal, every tree with $p_T>0$ in this decomposition is optimal; otherwise their average cost would exceed $\OPT_D$. The support selection above therefore recovers an optimal tree.
\end{proof}

\begin{proof}{Proof of Corollary~\ref{cor:rmp}.}
Setting $x_{q\omega}=0$ for excluded configurations defines a face of $P$. After zero extension, the restricted master is exactly this face: only separator equations that reduce to $0=0$ are omitted. In \eqref{eq:ecdecomposition}, a zero coordinate excludes every tree using that configuration, whereas every tree using only retained configurations belongs to the face. Thus the face equals $\conv\{x^T:T\in\Tset,\ T|_{C_q}\in\widehat\Om_q\text{ for all }q\in\Q\}$.
It is empty if no compatible tree remains; otherwise all of its extreme points are integral, and every linear objective attains an optimum at one of them.
\end{proof}

\subsection{Depths two and three}
In the following shallow formulations, infeasible rule tuples are omitted, equivalently by fixing their variables to zero. For complete split-only trees with locally optimized terminal labels, a depth-two configuration consists of root rule $f$ and child rule $g$ on branch $b\in\{0,1\}$. With $c^b_{fg}$ denoting its minimized local cost, the formulation is
\begin{equation}\label{eq:jt2}
\begin{aligned}
 \min_{x\geq0}\quad &\sum_{b\in\{0,1\}}\sum_{f,g\in\F}c^b_{fg}x^b_{fg}\\
 \text{s.t.}\quad
 &\sum_{f,g\in\F}x^b_{fg}=1 &&b\in\{0,1\},\\
 &\sum_{g\in\F}x^0_{fg}=\sum_{g\in\F}x^1_{fg} &&f\in\F.
\end{aligned}
\end{equation}
This is the depth-two structure studied by \citet{organ2026rolling}. At depth three, let $x^{ab}_{fgh}$ select rules $f$, $g$, and $h$ at positions $\epsilon$, $a$, and $ab$. In addition to one normalization equation for each $ab\in\{00,01,10,11\}$, the separator equations are
\begin{equation}\label{eq:jt3}
\begin{aligned}
 \sum_hx^{00}_{fgh}&=\sum_hx^{01}_{fgh} &&f,g\in\F,\\
 \sum_{g,h}x^{01}_{fgh}&=\sum_{g,h}x^{10}_{fgh} &&f\in\F,\\
 \sum_hx^{10}_{fgh}&=\sum_hx^{11}_{fgh} &&f,g\in\F.
\end{aligned}
\end{equation}
{
The outer equations match joint root--child actions; the middle equation
matches the root. Under matched rules, feasibility filters, and losses, these
are the four configuration groups in Appendix~A of the 2023 preprint of
\citet{organ2026rolling}, in the order $(00,01,10,11)$. Theorem~\ref{thm:integrality}
allows arbitrary depth, early stopping, and ancestor-local feasibility.
\par}

For reference, the split-only coefficients at depths two and three are
\begin{equation}\label{eq:d2cost}
 c^b_{fg}=\frac12\lambda_{\epsilon,f}+\lambda_{b,g}
 +\sum_{d=0}^1\min_{k\in\K}
 \sum_{\substack{i\in I:\;z_{if}=b,\;z_{ig}=d}}
 \frac{1}{n}\mathbf 1\{y_i\neq k\}.
\end{equation}
and
\begin{equation}\label{eq:d3cost}
 c^{ab}_{fgh}=\frac14\lambda_{\epsilon,f}
 +\frac12\lambda_{a,g}+\lambda_{ab,h}
 +\sum_{d=0}^1\min_{k\in\K}
 \sum_{\substack{i\in I:\;z_{if}=a,\;z_{ig}=b,\;z_{ih}=d}}
 \frac{1}{n}\mathbf 1\{y_i\neq k\}.
\end{equation}
They follow directly from the multiplicities in \eqref{eq:localcost}; terminal labels can be minimized locally because the matched split-only models do not couple terminal predictions.
% ===== END inlined file: sections/appendix_theory =====

% ===== BEGIN inlined file: sections/exact_reductions_proofs =====
\section{Proofs for Exact Reductions}
\label{ec:exactreductions}

For private depth $h$, put $a=D-h$ and let $A_t$ be the strict ancestors of $t\in\Q_h$. Write $\mathcal U_h(t,\sigma)$ for feasible private completions under separator assignment $\sigma$, and $J_t(U;\sigma)$ for their prediction loss and split penalties. The allocated ancestor cost is $g_t(\sigma)$ in \eqref{eq:ecancestorcost}.

\begin{proof}{Proof of Proposition~\ref{prop:signature_compression}.}
Let $T$ denote the aggregation map in
\eqref{eq:signature_aggregation}. If $x\in P$, summing the JT-LP
normalization row over the partition
$\{\omega\in\Om_q:\gamma_q(\omega)=\eta\}_{\eta\in\Gamma_q}$ gives
\eqref{eq:signature_compressed_norm}. For an edge $e=\{q,r\}$ and assignment
$\sigma\in\Sigma_e$,
\[
 \sum_{\substack{\eta\in\Gamma_q:\eta_e=\sigma}}(Tx)_{q\eta}
 =\sum_{\substack{\omega\in\Om_q:\omega|_{S_e}=\sigma}}x_{q\omega}.
\]
The analogous identity holds at $r$, so \eqref{eq:jtsep} gives
\eqref{eq:signature_compressed_sep}. Hence $T(P)$ is contained in the
compressed region.

Conversely, let $y$ satisfy \eqref{eq:signature_compressed_lp} and define
\[
 x_{q\omega}=\begin{cases}
 y_{q\eta},&\omega=\omega_q^\star(\eta)\text{ for some }\eta\in\Gamma_q,\\
 0,&\text{otherwise}.
 \end{cases}
\]
The stored representative has signature $\eta$, so the normalization and
separator marginals of $x$ equal those of $y$. Thus $x\in P$ and $Tx=y$,
which proves equality of the feasible-region image. By
Theorem~\ref{thm:integrality}, $P$ is the convex hull of tree-selection
vectors. Their images under $T$ select one signature per cluster and are
integral. The compressed region is therefore also integral.

For arbitrary $x\in P$, the definition of $\bar c_q$ gives
\[
 \sum_{q,\eta}\bar c_q(\eta)(Tx)_{q\eta}
 \leq \sum_{q,\omega}c_q(\omega)x_{q\omega}.
\]
For the lift above, equality holds because
$c_q(\omega_q^\star(\eta))=\bar c_q(\eta)$. Applying the two maps to optimal
solutions proves equality of the optimal values. If $y$ is integral, exactly
one signature is selected at each cluster and the displayed lift selects
exactly one representative there, so the lift is integral and has the same
objective value.
\end{proof}

\begin{proof}{Proof of Proposition~\ref{prop:reduction_relationship}.}
Fix $t\in\Q_h$ and a feasible external separator assignment $\sigma$ for the
block $\mathcal B_t$. Running intersection and the ancestor-local predicates
give a bijection between the completions in $\mathcal U_h(t,\sigma)$ and the
mutually consistent original configurations in $\mathcal B_t$ that agree with
$\sigma$ on the external separators. This remains true when an ancestor
predicts, in which case all private positions are inactive.

Each strict ancestor $v$ occurs in all $2^{h-1}$ clusters of the block, giving total coefficient $2^{h-1}/2^{D-1-|v|}=2^{-(a-|v|)}$. Each private position has all its occurrences in this block, so its coefficients sum to one. The block cost is therefore $g_t(\sigma)+J_t(U;\sigma)$; minimizing gives \eqref{eq:block_contraction_identity}, including the empty-domain convention.

The external signature records all strict ancestors: for $a\geq1$, a sibling group shares them; for $a=0$, the signature is empty. Compression thus minimizes exactly the private cost. For $h=1$, the block is one original cluster, giving the SC identity. Recompression is idempotent while its interface is unchanged.
\end{proof}

\begin{proof}{Proof of Proposition~\ref{prop:leaf_cluster_elimination}.}
Because $q$ is a leaf, its signature is its assignment $\sigma$ on $S_e$.
The separator equations therefore imply
\eqref{eq:leaf_variable_recovery}. Summing those identities over
$\sigma\in\Sigma_e$ and using the normalization at $r$ shows that the
normalization at $q$ is redundant.

Projection deletes only the coordinates and constraints of $q$. Conversely, \eqref{eq:leaf_variable_recovery} uniquely lifts a reduced feasible vector: feasibility filtering supplies an admissible state at $q$ for each separator assignment used at $r$. The lift satisfies the deleted separator and normalization rows, so the two maps are inverse.

Finally, substitution gives
\begin{align*}
 &\sum_{\eta\in\Gamma_r}\bar c_r(\eta)y_{r\eta}
   +\sum_{\sigma\in\Gamma_q}\bar c_q(\sigma)y_{q\sigma}\\
 &\quad=\sum_{\eta\in\Gamma_r}
   [\bar c_r(\eta)+\bar c_q(\eta_e)]y_{r\eta}
 =\sum_{\eta\in\Gamma_r}\widetilde c_r(\eta)y_{r\eta}.
\end{align*}
Thus the bijection preserves objective values. It also preserves integral
vectors. Combining the recovered signatures with the stored representatives
from Proposition~\ref{prop:signature_compression} recovers an optimal JT-LP
solution and hence an optimal tree by Theorem~\ref{thm:integrality}.
\end{proof}

A later leaf may retain components from eliminated neighbors. Recompressing its current potential on its remaining separator gives
\begin{equation}\label{eq:reduction_leaf_message}
 M_{q\to r}(\sigma)=
 \min_{\substack{\eta\in\Gamma_q:\eta_e=\sigma}}
 \left\{\bar c_q(\eta)+
 \sum_{k\in N_{\mathcal J}(q)\setminus\{r\}}
 M_{k\to q}(\eta_{\{k,q\}})\right\}.
\end{equation}
Proposition~\ref{prop:leaf_cluster_elimination} applies at each step, proving exactness by induction. Equation~\eqref{eq:reduction_leaf_message} is the compressed form of the recursion in Section~\ref{sec:mp}.

{
\paragraph{Derivation of the depth-three counts.}
Before compression, each of the four clusters has $F^3$ split configurations.
Minimizing its private deepest split conditional on the two ancestor splits
leaves $F^2$ signatures, so the first two counts are $4F^3$ and $4F^2$. Eliminating
$C_{00}$ and $C_{11}$ leaves the two $F^2$ variable arrays, giving $2F^2$ variables.

There are $F$ root-consistency rows and one retained normalization row. To
verify independence, take a linear combination with coefficient $\lambda_a$
on the consistency row for $a$ and coefficient $\mu$ on the normalization
row. The coefficient of each $y^R_{ac}$ is $-\lambda_a$, so all
$\lambda_a=0$. The coefficient of each $y^L_{ab}$ is then $\mu$, so
$\mu=0$. Hence the $F+1$ rows are linearly independent.
\par}

% ===== BEGIN inlined file: sections/appendix_contraction =====
\subsection{Configuration Counts and Subtree Contraction}
\label{ec:pricingstructure}
With all $F$ rules allowed repeatedly and no extra feasibility filters, a base configuration either splits at its $D$ internal positions and labels both terminals, or predicts after $j<D$ splits and leaves the remaining positions inactive. Hence
\begin{equation}\label{eq:ecrawcount}
 |\Om_q|=F^D K^2+K\sum_{j=0}^{D-1}F^j.
\end{equation}
For a split-only model with terminal labels optimized locally, the
first term becomes $F^D$ and the second is absent, giving
\eqref{eq:sizesplit}. If repeated rules on a path are forbidden,
replace $F^j$ by the falling factorial
$(F)_j=F(F-1)\cdots(F-j+1)$, with $(F)_0=1$ and $(F)_j=0$ for $j>F$.
Additional local restrictions can only reduce these counts.
The same partition by the first prediction position proves
\eqref{eq:prefixcount}; without repeated rules its unfiltered count is
$(F)_a+K\sum_{j=0}^{a-1}(F)_j$. Candidate arrays may use the larger
repeated-rule index set and exclude infeasible entries during pricing; array
capacity therefore need not equal the number of feasible columns.

Conditional costs follow the recurrence in Section~\ref{sec:conditional_evaluation}. The next result shows that one minimizing private completion per ancestor assignment preserves the optimal value.

\begin{proof}{Proof of Proposition~\ref{prop:contraction}.}
The enlarged cluster is $C_t^{(h)}=A_t\cup B_t^{(h)}$, where $B_t^{(h)}$ is the private subtree. Clusters containing an upper position are consecutive in lexicographic order; each private position belongs to one cluster. Running intersection therefore holds, and every private decision occurs with the ancestors needed to evaluate its loss and feasibility.

For $|v|<a$, the multiplicity is $m_v^{(h)}=2^{a-|v|}$; private
positions have multiplicity one. The allocated ancestor term is
\begin{equation}\label{eq:ecancestorcost}
 g_t(\sigma)=\sum_{v\in A_t}
                     \frac{\ell_v(\sigma)}{2^{a-|v|}}.
\end{equation}
The same counting argument as in Proposition~\ref{prop:costdecomp}
shows that the enlarged cluster costs sum to $J(T)$ for any feasible
tree. The configuration--tree correspondence in Electronic Companion~\ref{ec:hull} applies to these enlarged clusters as well. Classical junction-tree exactness \citep[Proposition~2.1]{wainwright2008graphical} therefore makes their complete normalized separator model integral.

At fixed $t$ and $\sigma$, all private completions have identical master coefficients. Replacing them by a minimizer of \eqref{eq:privatecost} preserves feasibility and cannot increase cost. Applied to an optimal full tree, this gives a representative selection costing at most $\OPT_D$.

Conversely, consistent representatives assemble into a feasible tree with the same cost, proving the reverse inequality. Retaining one minimizer per ancestor assignment is a column restriction of the enlarged integral model; Corollary~\ref{cor:rmp} therefore proves integrality. The support traversal in Theorem~\ref{thm:integrality}, followed by insertion of the stored subtrees, recovers an optimal tree even from a fractional optimum.

{
Cross-group constraints can be accommodated by enlarging the interface to
retain the private decisions on which they depend.
\par}

\end{proof}
% ===== END inlined file: sections/appendix_contraction =====
% ===== END inlined file: sections/exact_reductions_proofs =====

% ===== BEGIN inlined file: sections/on_demand_evaluation =====
\section{Certificates and Conditional-Cost Bounds}\label{sec:certificate}\label{ec:cg}
{
Subtree searches can provide useful bounds before reaching optimality.
This section explains how lower bounds and feasible subtree solutions guide
further evaluation and provide certificates for the complete tree.
\par}

\subsection{Cost Intervals and Adaptive Message Passing}
{
Extend $\Xi_t$ to include candidate assignments with unresolved feasibility;
those without a feasible completion have exact cost $+\infty$. For every
state, including those absent from the RMP, maintain
\par}

\begin{equation}\label{eq:costinterval}
 \underline\kappa_t(\sigma)=g_t(\sigma)+\underline V_h(t,\sigma)
 \ \leq\ \kappa_t(\sigma)\ \leq\
 \overline\kappa_t(\sigma)=g_t(\sigma)+\overline V_h(t,\sigma).
\end{equation}
{
Retain a feasible subtree for each finite upper cost. Combine bounds by
taking the maximum lower and minimum upper value; exact evaluation closes
the interval. Section~\ref{ec:costbounds} gives the initial bounds.
\par}

\paragraph{Adaptive message passing.}
Apply the recursion in Section~\ref{sec:mp} to all lower costs and to the feasible upper costs. Refine unresolved states in a minimizing lower-cost selection until the certified gap closes. The following result justifies this procedure.

\begin{proposition}[Interval certificate and finite refinement]\label{prop:adaptivecertificate}
Let $\mathcal A$ contain all separator-consistent candidate-state selections, with infeasible completions costing $+\infty$. Assume $\Tset\neq\varnothing$, $\underline\kappa\leq\kappa\leq\overline\kappa$ over the full domain, and a feasible realization for every finite upper cost. Put
\begin{equation}\label{eq:intervalcertificate}
 L=\min_{s\in\mathcal A}\sum_t\underline\kappa_t(s_t),\qquad
 U=\min_{s\in\mathcal A}\sum_t\overline\kappa_t(s_t),\qquad
 s^-\in\arg\min_{s\in\mathcal A}\sum_t\underline\kappa_t(s_t).
\end{equation}
Assume the lower costs are bounded below and an initial feasible tree is available. Then $L\leq\OPT_D\leq U$, and, whenever the upper costs along $s^-$ are finite,
\begin{equation}\label{eq:selectedwidth}
 0\leq U-L\leq\sum_t\bigl[\overline\kappa_t(s^-_t)-\underline\kappa_t(s^-_t)\bigr].
\end{equation}
An independently verified incumbent may further reduce $U$. In exact arithmetic, resolving at least one unresolved selected state whenever the gap is positive terminates with an optimal tree after at most $N_{\rm unresolved}$ resolutions. Here $N_{\rm unresolved}$ counts initially unresolved states, exact values are retained, and resolution includes proving infeasibility. This bound does not count partial oracle calls, message sweeps, or runtime.
\end{proposition}
{
Recompute both message arrays after refinement; the proof is in
Section~\ref{ec:finiteproofs}.
\par}

\subsection{Column Generation with Cost Bounds}
{
The RMP retains feasible representatives with their upper costs. Lower costs
on the full candidate domain yield a pricing certificate.
\par}
{
For the current RMP dual, extend absent separator multipliers by zero and set
\begin{equation}\label{eq:eccontractlower}
 \underline r_t=\min_{\sigma\in\Xi_t}
 \left\{\underline\kappa_t(\sigma)-\alpha_t
 -\sum_{e\in E(t)}s_{te}\pi_{e,\sigma|_{S_e}}\right\}.
\end{equation}
The minimum covers all states, evaluated or not; block bounds combine as in \eqref{eq:ecparallelbound}. With $\gamma_t=\min\{0,\underline r_t\}$, the inequality $\underline\kappa_t\leq\kappa_t$ makes $(\alpha+\gamma,\pi)$ exact-cost dual feasible. Thus
\begin{equation}\label{eq:eccontractcertificate}
 \sum_{t\in\Q_h}\alpha_t+
 \sum_{t\in\Q_h}\min\{0,\underline r_t\}\leq\OPT_D.
\end{equation}
The RMP uses feasible upper costs; its optimal dual objective is $\sum_t\alpha_t$. Equation~\eqref{eq:eccontractcertificate} remains a lower bound when retained costs are unresolved.

\par}
{
For configuration-cost bounds $\underline c_q(\omega)\leq c_q(\omega)$,
message passing gives a second certificate:
\par}
\begin{equation}\label{eq:jtbound}
 \LB_{\rm JT}=\min_{x\in P}\sum_{q\in\Q}\sum_{\omega\in\Om_q}
 \underline c_q(\omega)x_{q\omega}\leq\OPT_D.
\end{equation}
{
For $\varepsilon$-optimality, exact clusterwise pricing with tolerance
$\tau$ requires $2^{D-h}\tau\leq\varepsilon$. Both certificates cover the
full domain, including states absent from the RMP.
\par}

{A state is promising when the expression in
\eqref{eq:eccontractlower} gives a negative reduced-cost lower bound.
Algorithm~\ref{alg:intervalcg} permits refinement both inside and outside the RMP.\par}

\setcounter{algorithm}{0}
\renewcommand{\thealgorithm}{EC.\arabic{algorithm}}
\renewcommand{\theHalgorithm}{EC.\arabic{algorithm}}
\begin{algorithm}[!htbp]

\caption{{JT-CG with subtree cost bounds}}\label{alg:intervalcg}
\footnotesize
\begin{tabularx}{\linewidth}{@{}r>{\raggedright\arraybackslash}X@{}}
&\textbf{Input:} Candidate domains $\Xi_t$, feasible tree $T^0$, tolerance $\varepsilon$, resource limit, and finite RMP interval $b$.\\
1&Initialize valid intervals on $\Xi_t$. Retain the feasible representatives of $T^0$ and all induced separator rows; set $\UB=J(T^0)$ and a valid $\LB$.\\
2&Solve the upper-cost RMP. Obtain its dual, extend absent row multipliers by zero, and recover a feasible tree to update $\UB$.\\
3&Bound pricing over the full domain under this dual. Update $\LB$ using \eqref{eq:eccontractcertificate}, lower-cost messages, and the depth bound. Stop if $\UB-\LB\leq\varepsilon$.\\
4&Evaluate promising absent states. Retain intervals from incomplete solves; add exact-cost improving columns and their induced rows. Remove proven infeasible states from further evaluation.\\
5&Refine promising unresolved retained states. Update intervals, feasible representatives, column objective coefficients, and $\UB$.\\
6&After at most $b$ batches, or when a new dual is needed, return to step 2; otherwise repeat steps 3--5 using the same dual. Retain each certificate with the dual and bounds used to compute it.\\
7&On a resource stop, return the incumbent tree and the last valid bounds.\\
\end{tabularx}
\end{algorithm}

{
Conditional-cost bounds are reusable for the same subproblem. Reduced-cost
bounds use a fixed dual and must be recomputed when it changes. Updating a
retained column changes the RMP objective, so its earlier optimum is no
longer the current RMP value; the associated global lower certificate remains
valid. The implementation retains columns and exact costs and re-solves each
batch, except for the D5 \texttt{fico} and \texttt{spambase} runs, which use
four batches.
\par}

\subsection{Proofs and Finite Termination}\label{ec:finiteproofs}

\begin{proof}{Proof of Proposition~\ref{prop:mpdual}.}
Choose a feasible-tree cost $\bar z$ and $H>\bar z$. In \eqref{eq:exactmessages}, replace an empty conditional minimum by $H$, leaving other recursion steps unchanged; denote these finite messages by $\widehat M^c$. Nonnegative local costs make every root assignment with an infeasible component cost at least $H$, whereas a feasible assignment costs $\bar z$. The modified root minimum remains $\OPT_D$.

Orient every edge from a child $q$ to its parent $p$. In the full-domain version of \eqref{eq:rmpdual}, the multiplier on this edge has sign $+1$ in the constraint for $q$ and sign $-1$ in the constraint for $p$. For every nonroot cluster $q$ and configuration $\omega\in\Om_q$, the message definition gives
\[
 \widehat M^c_{q\to p}(\omega|_{S_{qp}})
 -\sum_{r\text{ child of }q}\widehat M^c_{r\to q}(\omega|_{S_{rq}})
 \leq c_q(\omega).
\]
This is the dual constraint at $q$ after setting $\alpha_q=0$ and $\pi_{e\sigma}=\widehat M^c_{q\to p}(\sigma)$. At the root, \eqref{eq:mproot} gives
\[
 \OPT_D-\sum_{r\text{ child of }q_0}
 \widehat M^c_{r\to q_0}(\omega|_{S_{rq_0}})
 \leq c_{q_0}(\omega),
 \qquad \omega\in\Om_{q_0}.
\]
Hence the proposed multipliers are dual feasible. Their objective is $\sum_q\alpha_q=\OPT_D$. Message backtracking supplies a feasible tree of the same value by \eqref{eq:mproot} and Theorem~\ref{thm:integrality}. Weak duality then makes the recovered tree and the displayed multipliers primal--dual optimal.
\end{proof}

\begin{proof}{Proof of Proposition~\ref{prop:pricingbound}.}
Strong duality gives $z_R=\sum_q\alpha_q$. Extend absent separator-row multipliers by zero, and set $\gamma_q=\min\{0,\underline r_q\}$ and $\widetilde\alpha_q=\alpha_q+\gamma_q$. For every full-domain column $(q,\omega)$,
\begin{align*}
 c_q(\omega)-\widetilde\alpha_q
 -\sum_{e\in E(q)}s_{qe}\pi_{e,\omega|_{S_e}}
 &=\widetilde c_q(\omega)-\gamma_q\\
 &\geq r_q-\gamma_q\\
 &\geq \underline r_q-\min\{0,\underline r_q\}\geq0.
\end{align*}
Thus $(\widetilde\alpha,\pi)$ is feasible for the dual of the complete master problem. Its objective is $z_R+\sum_q\gamma_q$, and weak duality proves \eqref{eq:cgbound}. If $\underline r_q=r_q\geq-\tau$ for every $q$, then \eqref{eq:cgbound} is at least $z_R-|\Q|\tau$. Corollary~\ref{cor:rmp} implies that the RMP optimum is attained by a feasible tree, so $\OPT_D\leq z_R$. This proves \eqref{eq:epsbound}. For a nonoptimal dual feasible solution, replace $z_R$ by $z_D=\sum_q\alpha_q$, giving \eqref{eq:cgdualgeneral}.
\end{proof}

Replacing exact costs by their lower bounds gives the recursion

\begin{equation}\label{eq:messages}
 M_{q\to p}(\sigma)=
 \min_{\substack{\omega\in\Om_q:\;\omega|_{S_{qp}}=\sigma}}
 \left\{\underline c_q(\omega)+
 \sum_{\substack{r:\;\{r,q\}\in E\\r\neq p}}
 M_{r\to q}(\omega|_{S_{rq}})\right\}.
\end{equation}
{
Here $S_{qp}=C_q\cap C_p$, with an empty minimum equal to $+\infty$.
Induction from the leaves identifies each message as the conditional minimum
on its component: a fixed configuration fixes the child separators, and
running intersection permits addition of the child minima. At the root this
evaluates \eqref{eq:jtbound}; cost monotonicity gives its lower-bound property.
Exact costs give \eqref{eq:exactmessages} and \eqref{eq:mproot}.
\par}

\begin{proof}{Proof of Proposition~\ref{prop:adaptivecertificate}.}
Cost monotonicity gives the lower bound, including infeasible candidate selections. Finite upper costs assemble into a feasible tree by Proposition~\ref{prop:contraction}, giving the upper bound. Evaluating that upper objective at $s^-$ proves \eqref{eq:selectedwidth}. If every selected state is exact, a finite feasible selection rules out any selected state of cost $+\infty$; the recovered cost then equals $L$. Thus a positive gap requires an unresolved selected state. Exact resolution reduces their finite count, proving the claimed termination bound. 
\end{proof}

A further bound applies when $\lambda_{v,f}=\lambda>0$. Let $\OPT_{d_0}$ denote the optimum among trees of depth at most $d_0<D$, and let $L_{d_0}\leq\OPT_{d_0}$. Every tree of depth greater than $d_0$ contains at least $d_0+1$ split nodes on one root-to-leaf path. Since its classification error is nonnegative,
\begin{equation}\label{eq:ecdepthbound}
 \min\{L_{d_0},(d_0+1)\lambda\}\leq\OPT_D.
\end{equation}
The two terms cover trees of depth at most $d_0$ and greater than $d_0$, respectively. The D4/D5 computations use $d_0=2$ and the certified lower bound from the matched depth-two problem, giving $\min\{L_2,3\lambda\}$.

\begin{proof}{Certification and finite termination.}
The bounds in Proposition~\ref{prop:pricingbound}, \eqref{eq:jtbound}, and \eqref{eq:ecdepthbound} remain valid when combined by their maximum. A feasible incumbent attaining $\UB$ satisfies $J(T)-\OPT_D\leq\UB-\LB$, establishing the stopping certificate.

For finite termination, assume exact arithmetic and pricing, no resource limit, and no column deletion. With $N=\sum_q|\Om_q|$ and $N_0$ initial columns, each negative reduced-cost column is absent from the RMP because its current columns have nonnegative reduced costs. Every nonterminal iteration therefore adds a new column, allowing at most $N-N_0$ such iterations. Once exact pricing finds none, the zero-extended dual is feasible for the complete master. The primal is also feasible because all induced rows are retained. Strong duality proves optimality. 
\end{proof}

{
\paragraph{Finite progress with interval costs.}
Assume finite domains, exact arithmetic, no resource limit or column deletion,
and retention of exact costs. Require complete exact pricing and the following
progress within finitely many batches whenever the gap is positive: resolve
an unresolved state, including proof of infeasibility, or add an absent
exact-cost improving column. Each state can be resolved once and each column
added once. Once the remaining costs and pricing are exact, a positive gap
implies an omitted improving column by Proposition~\ref{prop:pricingbound}.
Thus these evaluation and scheduling assumptions give finite termination.
\par}

\subsection{Conditional-Subtree Bounds}\label{ec:costbounds}
The D4/D5 bounds use weights $1/n$, common split penalty $\lambda\geq0$, early stopping, unrestricted prediction labels, minimum leaf support, and private depth $h=3$. For routed observations $R$, define the prediction and conflict losses
\begin{equation}\label{eq:ecinitiallosses}
\begin{aligned}
 E(R)&=\frac{|R|-\max_{k\in\K}|\{i\in R:y_i=k\}|}{n},\\
 C(R)&=\frac1n\sum_G\left(|G|-\max_{k\in\K}|\{i\in G:y_i=k\}|\right).
\end{aligned}
\end{equation}
Groups $G$ share identical available-feature vectors and hence follow the same route in every tree, making $C(R)$ an error lower bound.

For an active subtree with sufficient leaf support and $h\geq1$, the initial lower and upper costs are
\begin{equation}\label{eq:ecinitialbound}
 \underline V_h^{\,0}(R)=\min\{E(R),\lambda+C(R)\},
 \qquad \overline V_h^{\,0}(R)=E(R).
\end{equation}
Prediction costs $E(R)$ and a split costs at least $\lambda+C(R)$; majority prediction attains the upper bound. At $h=0$, both bounds equal $E(R)$. Insufficient leaf support makes an active state infeasible, with cost $+\infty$; inactive subtrees cost zero.

Let $N_{(1)}(R)\geq\cdots\geq N_{(K)}(R)$ be the class counts in $R$, including zero counts, where $K=|\K|$. A subtree with $b$ splits has $b+1$ leaves and can correctly classify observations from at most $b+1$ distinct classes. The optional bound is
\begin{equation}\label{eq:ecclassbound}
 B_h(R)=\min_{0\leq b\leq\min\{2^h-1,K-1\}}
 \left\{b\lambda+\max\left[C(R),
 \frac{|R|-\sum_{j=1}^{b+1}N_{(j)}(R)}{n}\right]\right\}.
\end{equation}
The maximum combines two potentially overlapping error bounds. For $b\geq K-1$, its class-count term is zero and additional splits cannot improve the expression since $\lambda\geq0$; the truncated minimum therefore bounds all feasible subtrees. When enabled, use $\max\{\underline V_h^{\,0}(R),B_h(R)\}$, with $h$ the remaining private depth.

{For binary classification, $B_h(R)$ reduces to the initial bound
in \eqref{eq:ecinitialbound}; the class-count bound is therefore enabled by
default only for multiclass data.}

Conditioning the lower-cost messages on root feature $f$ gives a valid bound $L(f)$ for trees using that split. If $L(f)\geq\UB$, the candidate cannot improve the incumbent, although its states remain in the lower-cost domain.

% ===== END inlined file: sections/on_demand_evaluation =====

\label{ec:technical-end}
\clearpage
\section{Supporting Computational Results}\label{app:results}

\subsection{Individual benchmark results}\label{app:benchmarkdetail}
Table~\ref{tab:benchmarkdetail} reports per-instance results for the exact methods. Status codes are O (optimal certificate), T (time limit), M (memory limit), and NI (no solver incumbent, with the CART fallback objective reported); dashes denote unavailable entries, including unrun combinations. For JT methods, the relative gap is $100(\UB-\LB)/\max\{10^{-10},|\UB|\}$; other methods retain their reported gaps. Runtimes shown as 0.00 seconds are positive runtimes below the rounding threshold. CART and HCG are summarized in the main text.

% ===== BEGIN inlined file: tables/benchmark_detail =====
\begingroup\fontsize{7}{8.5}\selectfont\setlength{\tabcolsep}{2pt}
\captionsetup{font=benchmarkcaption}
\begin{longtable}{rllrrrc@{\hspace{2pt}}rrrc}
\caption{Individual exact-method results for depths two through five at $\lambda\in\{0,0.01\}$. $J$ includes split penalties; runtime is in seconds.}\label{tab:benchmarkdetail}\\
\toprule
\multicolumn{3}{c}{} & \multicolumn{4}{c}{$\lambda=0$} & \multicolumn{4}{c}{$\lambda=0.01$}\\\cmidrule(lr){4-7}\cmidrule(lr){8-11}$D$ & Dataset & Method & $J$ & \shortstack{Gap\\(\%)} & \shortstack{Runtime\\(s)} & Status & $J$ & \shortstack{Gap\\(\%)} & \shortstack{Runtime\\(s)} & Status\\
\midrule
\endfirsthead
\multicolumn{11}{l}{\tablename~\thetable\ (continued)}\\
\toprule
\multicolumn{3}{c}{} & \multicolumn{4}{c}{$\lambda=0$} & \multicolumn{4}{c}{$\lambda=0.01$}\\\cmidrule(lr){4-7}\cmidrule(lr){8-11}$D$ & Dataset & Method & $J$ & \shortstack{Gap\\(\%)} & \shortstack{Runtime\\(s)} & Status & $J$ & \shortstack{Gap\\(\%)} & \shortstack{Runtime\\(s)} & Status\\
\midrule
\endhead
\midrule
\multicolumn{11}{r}{Continued on next page}\\
\endfoot
\bottomrule
\endlastfoot
2 & avila & BOCT & 0.4746 & 98.43 & 600.11 & T & 0.5046 & 93.23 & 600.11 & T\\
 &  & DL8.5 & 0.4664 & --- & 0.32 & O & --- & --- & --- & ---\\
 &  & GOSDT & 0.4664 & 0.00 & 0.40 & O & 0.4964 & 0.00 & 0.41 & O\\
 &  & MurTree & 0.4664 & --- & 0.32 & O & 0.4964 & --- & 0.32 & O\\
 &  & STreeD & 0.4664 & 0.00 & 1.00 & O & 0.4964 & 0.00 & 1.01 & O\\
 &  & Branches & 0.4664 & 0.00 & 0.95 & O & 0.4964 & 0.00 & 1.03 & O\\
 &  & JT-LP & 0.4664 & 0.00 & 0.02 & O & 0.4964 & 0.00 & 0.02 & O\\
 &  & JT-CG & 0.4664 & 0.00 & 0.01 & O & 0.4964 & 0.00 & 0.01 & O\\
 &  & JT-MP & 0.4664 & 0.00 & 0.04 & O & 0.4964 & 0.00 & 0.04 & O\\
\addlinespace[1.5pt]
2 & banknote & BOCT & 0.0794 & 0.00 & 59.23 & O & 0.1094 & 0.00 & 25.10 & O\\
 &  & DL8.5 & 0.0794 & --- & 0.01 & O & --- & --- & --- & ---\\
 &  & GOSDT & 0.0794 & 0.00 & 0.01 & O & 0.1094 & 0.00 & 0.01 & O\\
 &  & MurTree & 0.0794 & --- & 0.02 & O & 0.1094 & --- & 0.02 & O\\
 &  & STreeD & 0.0794 & 0.00 & 0.88 & O & 0.1094 & 0.00 & 0.88 & O\\
 &  & Branches & 0.0794 & 0.00 & 0.02 & O & 0.1094 & 0.00 & 0.02 & O\\
 &  & JT-LP & 0.0794 & 0.00 & 0.01 & O & 0.1094 & 0.00 & 0.01 & O\\
 &  & JT-CG & 0.0794 & 0.00 & 0.00 & O & 0.1094 & 0.00 & 0.00 & O\\
 &  & JT-MP & 0.0794 & 0.00 & 0.02 & O & 0.1094 & 0.00 & 0.02 & O\\
\addlinespace[1.5pt]
2 & compas & BOCT & 0.2757 & 100.00 & 600.06 & T & 0.2979 & 91.21 & 600.05 & T\\
 &  & DL8.5 & 0.2716 & --- & 0.09 & O & --- & --- & --- & ---\\
 &  & GOSDT & 0.2716 & 0.00 & 0.06 & O & 0.2953 & 0.00 & 0.06 & O\\
 &  & MurTree & 0.2716 & --- & 0.20 & O & 0.2953 & --- & 0.18 & O\\
 &  & STreeD & 0.2716 & 0.00 & 0.90 & O & 0.2953 & 0.00 & 0.91 & O\\
 &  & Branches & 0.2716 & 0.00 & 0.14 & O & 0.2953 & 0.00 & 0.16 & O\\
 &  & JT-LP & 0.2716 & 0.00 & 0.02 & O & 0.2953 & 0.00 & 0.02 & O\\
 &  & JT-CG & 0.2716 & 0.00 & 0.01 & O & 0.2953 & 0.00 & 0.01 & O\\
 &  & JT-MP & 0.2716 & 0.00 & 0.03 & O & 0.2953 & 0.00 & 0.02 & O\\
\addlinespace[1.5pt]
2 & diabetic & BOCT & 0.4349 & 100.00 & 667.36 & T & 0.4649 & 97.85 & 600.90 & T\\
 &  & DL8.5 & 0.4294 & --- & 7.59 & O & --- & --- & --- & ---\\
 &  & GOSDT & 0.4294 & 0.00 & 7.36 & O & 0.4452 & 0.00 & 9.72 & O\\
 &  & MurTree & 0.4294 & --- & 4.73 & O & 0.4452 & --- & 4.28 & O\\
 &  & STreeD & 0.4294 & 0.00 & 2.85 & O & 0.4452 & 0.00 & 2.96 & O\\
 &  & Branches & 0.4294 & 0.00 & 72.01 & O & 0.4452 & 0.00 & 110.01 & O\\
 &  & JT-LP & 0.4294 & 0.00 & 0.22 & O & 0.4452 & 0.00 & 0.21 & O\\
2 & diabetic & JT-CG & 0.4294 & 0.00 & 0.13 & O & 0.4452 & 0.00 & 0.13 & O\\
 &  & JT-MP & 0.4294 & 0.00 & 0.46 & O & 0.4452 & 0.00 & 0.46 & O\\
\addlinespace[1.5pt]
2 & fico & BOCT & 0.3020 & 100.00 & 606.28 & T & 0.3320 & 94.94 & 600.23 & T\\
 &  & DL8.5 & 0.2954 & --- & 0.35 & O & --- & --- & --- & ---\\
 &  & GOSDT & 0.2954 & 0.00 & 0.19 & O & 0.3173 & 0.00 & 0.19 & O\\
 &  & MurTree & 0.2954 & --- & 0.32 & O & 0.3173 & --- & 0.28 & O\\
 &  & STreeD & 0.2954 & 0.00 & 1.00 & O & 0.3173 & 0.00 & 0.98 & O\\
 &  & Branches & 0.2954 & 0.00 & 0.83 & O & 0.3173 & 0.00 & 0.85 & O\\
 &  & JT-LP & 0.2954 & 0.00 & 0.02 & O & 0.3173 & 0.00 & 0.02 & O\\
 &  & JT-CG & 0.2954 & 0.00 & 0.01 & O & 0.3173 & 0.00 & 0.01 & O\\
 &  & JT-MP & 0.2954 & 0.00 & 0.04 & O & 0.3173 & 0.00 & 0.03 & O\\
\addlinespace[1.5pt]
2 & give & BOCT & 0.0649 & 100.00 & 600.53 & T & 0.0768 & 79.94 & 600.58 & T\\
 &  & DL8.5 & 0.0649 & --- & 0.53 & O & --- & --- & --- & ---\\
 &  & GOSDT & 0.0649 & 0.00 & 4.32 & O & 0.0668 & 0.00 & 6.41 & O\\
 &  & MurTree & 0.0649 & --- & 1.34 & O & 0.0668 & --- & 1.33 & O\\
 &  & STreeD & 0.0649 & 0.00 & 1.40 & O & 0.0668 & 0.00 & 1.50 & O\\
 &  & Branches & 0.0649 & 0.00 & 0.84 & O & 0.0668 & 0.00 & 0.95 & O\\
 &  & JT-LP & 0.0649 & 0.00 & 0.04 & O & 0.0668 & 0.00 & 0.04 & O\\
 &  & JT-CG & 0.0649 & 0.00 & 0.03 & O & 0.0668 & 0.00 & 0.03 & O\\
 &  & JT-MP & 0.0649 & 0.00 & 0.10 & O & 0.0668 & 0.00 & 0.10 & O\\
\addlinespace[1.5pt]
2 & htru2 & BOCT & 0.0228 & 100.00 & 600.08 & T & 0.0522 & 61.37 & 600.10 & T\\
 &  & DL8.5 & 0.0228 & --- & 0.09 & O & --- & --- & --- & ---\\
 &  & GOSDT & 0.0228 & 0.00 & 0.09 & O & 0.0422 & 0.00 & 0.07 & O\\
 &  & MurTree & 0.0228 & --- & 0.22 & O & 0.0422 & --- & 0.20 & O\\
 &  & STreeD & 0.0228 & 0.00 & 0.94 & O & 0.0422 & 0.00 & 0.95 & O\\
 &  & Branches & 0.0228 & 0.00 & 0.19 & O & 0.0422 & 0.00 & 0.22 & O\\
 &  & JT-LP & 0.0228 & 0.00 & 0.01 & O & 0.0422 & 0.00 & 0.02 & O\\
 &  & JT-CG & 0.0228 & 0.00 & 0.01 & O & 0.0422 & 0.00 & 0.01 & O\\
 &  & JT-MP & 0.0228 & 0.00 & 0.03 & O & 0.0422 & 0.00 & 0.03 & O\\
\addlinespace[1.5pt]
2 & letter & BOCT & 0.8794 & 99.26 & 600.11 & T & 0.9094 & 96.69 & 600.11 & T\\
 &  & DL8.5 & 0.8558 & --- & 0.81 & O & --- & --- & --- & ---\\
 &  & GOSDT & 0.8558 & 0.00 & 1.01 & O & 0.8858 & 0.00 & 1.01 & O\\
 &  & MurTree & 0.8558 & --- & 0.39 & O & 0.8858 & --- & 0.38 & O\\
 &  & STreeD & 0.8558 & 0.00 & 1.05 & O & 0.8858 & 0.00 & 1.08 & O\\
 &  & Branches & 0.8558 & 0.00 & 1.24 & O & 0.8858 & 0.00 & 1.29 & O\\
 &  & JT-LP & 0.8558 & 0.00 & 0.04 & O & 0.8858 & 0.00 & 0.05 & O\\
 &  & JT-CG & 0.8558 & 0.00 & 0.02 & O & 0.8858 & 0.00 & 0.02 & O\\
 &  & JT-MP & 0.8558 & 0.00 & 0.05 & O & 0.8858 & 0.00 & 0.05 & O\\
\addlinespace[1.5pt]
2 & skin & BOCT & 0.1075 & 94.32 & 600.74 & T & 0.1375 & 81.30 & 600.67 & T\\
 &  & DL8.5 & 0.0794 & --- & 0.23 & O & --- & --- & --- & ---\\
2 & skin & GOSDT & 0.0794 & 0.00 & 0.42 & O & 0.1060 & 0.00 & 0.41 & O\\
 &  & MurTree & 0.0794 & --- & 1.21 & O & 0.1060 & --- & 1.12 & O\\
 &  & STreeD & 0.0794 & 0.00 & 5.20 & O & 0.1060 & 0.00 & 1.31 & O\\
 &  & Branches & 0.0794 & 0.00 & 0.22 & O & 0.1060 & 0.00 & 0.28 & O\\
 &  & JT-LP & 0.0794 & 0.00 & 0.03 & O & 0.1060 & 0.00 & 0.03 & O\\
 &  & JT-CG & 0.0794 & 0.00 & 0.02 & O & 0.1060 & 0.00 & 0.02 & O\\
 &  & JT-MP & 0.0794 & 0.00 & 0.08 & O & 0.1060 & 0.00 & 0.08 & O\\
\addlinespace[1.5pt]
2 & spambase & BOCT & 0.1500 & 97.29 & 600.10 & T & 0.1856 & 85.17 & 600.13 & T\\
 &  & DL8.5 & 0.1313 & --- & 0.22 & O & --- & --- & --- & ---\\
 &  & GOSDT & 0.1313 & 0.00 & 0.09 & O & 0.1613 & 0.00 & 0.08 & O\\
 &  & MurTree & 0.1313 & --- & 0.17 & O & 0.1613 & --- & 0.17 & O\\
 &  & STreeD & 0.1313 & 0.00 & 0.90 & O & 0.1613 & 0.00 & 0.92 & O\\
 &  & Branches & 0.1313 & 0.00 & 0.53 & O & 0.1613 & 0.00 & 0.52 & O\\
 &  & JT-LP & 0.1313 & 0.00 & 0.01 & O & 0.1613 & 0.00 & 0.02 & O\\
 &  & JT-CG & 0.1313 & 0.00 & 0.01 & O & 0.1613 & 0.00 & 0.01 & O\\
 &  & JT-MP & 0.1313 & 0.00 & 0.03 & O & 0.1613 & 0.00 & 0.02 & O\\
\addlinespace[1.5pt]
2 & transactions & BOCT & 0.0158 & 100.00 & 606.75 & T & 0.0458 & 100.00 & 607.36 & T\\
 &  & DL8.5 & 0.0158 & --- & 12.13 & O & --- & --- & --- & ---\\
 &  & GOSDT & 0.0158 & 0.00 & 22.71 & O & 0.0158 & 0.00 & 5.72 & O\\
 &  & MurTree & 0.0158 & --- & 15.26 & O & 0.0158 & --- & 30.38 & O\\
 &  & STreeD & 0.0158 & 0.00 & 8.15 & O & 0.0158 & 0.00 & 8.93 & O\\
 &  & Branches & 0.0158 & 0.00 & 17.06 & O & 0.0158 & 0.00 & 1.00 & O\\
 &  & JT-LP & 0.0158 & 0.00 & 0.57 & O & 0.0158 & 0.00 & 0.54 & O\\
 &  & JT-CG & 0.0158 & 0.00 & 0.52 & O & 0.0158 & 0.00 & 0.50 & O\\
 &  & JT-MP & 0.0158 & 0.00 & 1.69 & O & 0.0158 & 0.00 & 1.64 & O\\
\addlinespace[1.5pt]
3 & avila & BOCT & 0.4553 & 100.00 & 600.15 & T & 0.5253 & 93.91 & 600.11 & T\\
 &  & DL8.5 & 0.4270 & --- & 24.49 & O & --- & --- & --- & ---\\
 &  & GOSDT & 0.4270 & 0.00 & 52.45 & O & 0.4832 & 0.00 & 47.60 & O\\
 &  & MurTree & 0.4270 & --- & 1.19 & O & 0.4832 & --- & 1.19 & O\\
 &  & STreeD & 0.4270 & 0.00 & 2.70 & O & 0.4832 & 0.00 & 2.67 & O\\
 &  & Branches & 0.4270 & 0.00 & 35.16 & O & 0.4832 & 0.00 & 25.98 & O\\
 &  & JT-LP & 0.4270 & 0.00 & 0.44 & O & 0.4832 & 0.00 & 0.44 & O\\
 &  & JT-CG & 0.4270 & 0.00 & 0.51 & O & 0.4832 & 0.00 & 0.47 & O\\
 &  & JT-MP & 0.4270 & 0.00 & 0.52 & O & 0.4832 & 0.00 & 0.52 & O\\
\addlinespace[1.5pt]
3 & banknote & BOCT & 0.0219 & 100.00 & 600.03 & T & 0.0859 & 53.44 & 600.02 & T\\
 &  & DL8.5 & 0.0219 & --- & 0.10 & O & --- & --- & --- & ---\\
 &  & GOSDT & 0.0219 & 0.00 & 0.09 & O & 0.0784 & 0.00 & 0.07 & O\\
 &  & MurTree & 0.0219 & --- & 0.02 & O & 0.0784 & --- & 0.02 & O\\
 &  & STreeD & 0.0219 & 0.00 & 0.87 & O & 0.0784 & 0.00 & 0.88 & O\\
 &  & Branches & 0.0219 & 0.00 & 0.34 & O & 0.0784 & 0.00 & 0.25 & O\\
 &  & JT-LP & 0.0219 & 0.00 & 0.01 & O & 0.0784 & 0.00 & 0.01 & O\\
 &  & JT-CG & 0.0219 & 0.00 & 0.00 & O & 0.0784 & 0.00 & 0.01 & O\\
 &  & JT-MP & 0.0219 & 0.00 & 0.02 & O & 0.0784 & 0.00 & 0.02 & O\\
\addlinespace[1.5pt]
3 & compas & BOCT & 0.2743 & 100.00 & 600.15 & T & 0.3069 & 93.48 & 600.09 & T\\
 &  & DL8.5 & 0.2643 & --- & 4.06 & O & --- & --- & --- & ---\\
 &  & GOSDT & 0.2643 & 0.00 & 4.57 & O & 0.2953 & 0.00 & 3.15 & O\\
 &  & MurTree & 0.2643 & --- & 0.52 & O & 0.2953 & --- & 0.50 & O\\
3 & compas & STreeD & 0.2643 & 0.00 & 1.00 & O & 0.2953 & 0.00 & 1.00 & O\\
 &  & Branches & 0.2643 & 0.00 & 6.87 & O & 0.2953 & 0.00 & 6.54 & O\\
 &  & JT-LP & 0.2643 & 0.00 & 0.05 & O & 0.2953 & 0.00 & 0.04 & O\\
 &  & JT-CG & 0.2643 & 0.00 & 0.02 & O & 0.2953 & 0.00 & 0.02 & O\\
 &  & JT-MP & 0.2643 & 0.00 & 0.04 & O & 0.2953 & 0.00 & 0.04 & O\\
\addlinespace[1.5pt]
3 & diabetic & BOCT & 0.4319 & 100.00 & 601.07 & NI & 0.4919 & 97.97 & 601.16 & T\\
 &  & DL8.5 & 0.4253 & --- & 600.19 & T & --- & --- & --- & ---\\
3 & diabetic & GOSDT & --- & --- & 11.32 & M & 0.4452 & 95.50 & 602.79 & T\\
 &  & MurTree & 0.4253 & --- & 20.06 & O & 0.4452 & --- & 19.34 & O\\
 &  & STreeD & 0.4253 & 0.00 & 7.14 & O & 0.4452 & 0.00 & 7.70 & O\\
 &  & Branches & 0.4263 & 100.00 & 141.63 & M & 0.4452 & 93.26 & 514.53 & M\\
 &  & JT-LP & 0.4253 & 0.00 & 8.98 & O & 0.4452 & 0.00 & 8.19 & O\\
 &  & JT-CG & 0.4253 & 0.00 & 1.95 & O & 0.4452 & 0.00 & 1.31 & O\\
 &  & JT-MP & 0.4253 & 0.00 & 2.19 & O & 0.4452 & 0.00 & 1.75 & O\\
\addlinespace[1.5pt]
3 & fico & BOCT & 0.2992 & 100.00 & 600.14 & T & 0.3585 & 96.81 & 600.10 & T\\
 &  & DL8.5 & 0.2844 & --- & 47.18 & O & --- & --- & --- & ---\\
 &  & GOSDT & 0.2844 & 0.00 & 40.21 & O & 0.3173 & 0.00 & 34.17 & O\\
 &  & MurTree & 0.2844 & --- & 1.96 & O & 0.3173 & --- & 1.92 & O\\
 &  & STreeD & 0.2844 & 0.00 & 1.74 & O & 0.3173 & 0.00 & 1.81 & O\\
 &  & Branches & 0.2880 & 100.00 & 108.78 & M & 0.3173 & 81.02 & 68.97 & M\\
 &  & JT-LP & 0.2844 & 0.00 & 0.59 & O & 0.3173 & 0.00 & 0.59 & O\\
 &  & JT-CG & 0.2844 & 0.00 & 0.18 & O & 0.3173 & 0.00 & 0.28 & O\\
 &  & JT-MP & 0.2844 & 0.00 & 0.26 & O & 0.3173 & 0.00 & 0.24 & O\\
\addlinespace[1.5pt]
3 & give & BOCT & 0.0644 & 100.00 & 600.78 & T & 0.0768 & 73.91 & 601.36 & T\\
 &  & DL8.5 & 0.0644 & --- & 20.68 & O & --- & --- & --- & ---\\
 &  & GOSDT & 0.0644 & 0.00 & 33.56 & O & 0.0668 & 0.00 & 14.07 & O\\
 &  & MurTree & 0.0644 & --- & 3.66 & O & 0.0668 & --- & 3.20 & O\\
 &  & STreeD & 0.0644 & 0.00 & 2.53 & O & 0.0668 & 0.00 & 2.59 & O\\
 &  & Branches & 0.0644 & 0.00 & 19.45 & O & 0.0668 & 0.00 & 21.21 & O\\
 &  & JT-LP & 0.0644 & 0.00 & 0.18 & O & 0.0668 & 0.00 & 0.17 & O\\
 &  & JT-CG & 0.0644 & 0.00 & 0.09 & O & 0.0668 & 0.00 & 0.06 & O\\
 &  & JT-MP & 0.0644 & 0.00 & 0.19 & O & 0.0668 & 0.00 & 0.16 & O\\
\addlinespace[1.5pt]
3 & htru2 & BOCT & 0.0228 & 100.00 & 600.13 & T & 0.0528 & 62.12 & 600.08 & T\\
 &  & DL8.5 & 0.0218 & --- & 4.63 & O & --- & --- & --- & ---\\
 &  & GOSDT & 0.0218 & 0.00 & 5.34 & O & 0.0422 & 0.00 & 0.40 & O\\
 &  & MurTree & 0.0218 & --- & 0.55 & O & 0.0422 & --- & 0.40 & O\\
 &  & STreeD & 0.0218 & 0.00 & 1.04 & O & 0.0422 & 0.00 & 1.00 & O\\
 &  & Branches & 0.0218 & 0.00 & 8.09 & O & 0.0422 & 0.00 & 4.52 & O\\
 &  & JT-LP & 0.0218 & 0.00 & 0.05 & O & 0.0422 & 0.00 & 0.05 & O\\
 &  & JT-CG & 0.0218 & 0.00 & 0.03 & O & 0.0422 & 0.00 & 0.02 & O\\
 &  & JT-MP & 0.0218 & 0.00 & 0.07 & O & 0.0422 & 0.00 & 0.03 & O\\
\addlinespace[1.5pt]
3 & letter & BOCT & 0.8048 & 100.00 & 600.11 & T & 0.8748 & 96.28 & 600.29 & T\\
 &  & DL8.5 & 0.7465 & --- & 67.17 & O & --- & --- & --- & ---\\
 &  & GOSDT & 0.7465 & 0.00 & 295.98 & O & 0.8165 & 0.00 & 152.77 & O\\
 &  & MurTree & 0.7465 & --- & 2.33 & O & 0.8165 & --- & 2.55 & O\\
 &  & STreeD & 0.7465 & 0.00 & 4.85 & O & 0.8165 & 0.00 & 5.07 & O\\
 &  & Branches & 0.7465 & 0.00 & 38.06 & O & 0.8165 & 0.00 & 44.71 & O\\
3 & letter & JT-LP & 0.7465 & 0.00 & 0.74 & O & 0.8165 & 0.00 & 0.76 & O\\
 &  & JT-CG & 0.7465 & 0.00 & 0.56 & O & 0.8165 & 0.00 & 0.53 & O\\
 &  & JT-MP & 0.7465 & 0.00 & 0.59 & O & 0.8165 & 0.00 & 0.56 & O\\
\addlinespace[1.5pt]
3 & skin & BOCT & 0.0520 & 89.00 & 602.24 & T & 0.1120 & 77.03 & 600.73 & T\\
 &  & DL8.5 & 0.0394 & --- & 1.53 & O & --- & --- & --- & ---\\
 &  & GOSDT & 0.0394 & 0.00 & 5.96 & O & 0.0894 & 0.00 & 2.56 & O\\
 &  & MurTree & 0.0394 & --- & 1.54 & O & 0.0894 & --- & 1.65 & O\\
3 & skin & STreeD & 0.0394 & 0.00 & 1.56 & O & 0.0894 & 0.00 & 1.70 & O\\
 &  & Branches & 0.0394 & 0.00 & 2.07 & O & 0.0894 & 0.00 & 1.97 & O\\
 &  & JT-LP & 0.0394 & 0.00 & 0.06 & O & 0.0894 & 0.00 & 0.06 & O\\
 &  & JT-CG & 0.0394 & 0.00 & 0.04 & O & 0.0894 & 0.00 & 0.03 & O\\
 &  & JT-MP & 0.0394 & 0.00 & 0.09 & O & 0.0894 & 0.00 & 0.09 & O\\
\addlinespace[1.5pt]
3 & spambase & BOCT & 0.1115 & 99.22 & 600.05 & T & 0.1789 & 87.35 & 600.06 & T\\
 &  & DL8.5 & 0.0963 & --- & 23.91 & O & --- & --- & --- & ---\\
 &  & GOSDT & 0.0963 & 0.00 & 38.65 & O & 0.1521 & 0.00 & 11.94 & O\\
 &  & MurTree & 0.0963 & --- & 1.08 & O & 0.1521 & --- & 1.30 & O\\
 &  & STreeD & 0.0963 & 0.00 & 1.04 & O & 0.1521 & 0.00 & 1.07 & O\\
 &  & Branches & 0.0963 & 0.00 & 64.76 & O & 0.1521 & 0.00 & 50.34 & O\\
 &  & JT-LP & 0.0963 & 0.00 & 0.48 & O & 0.1521 & 0.00 & 0.49 & O\\
 &  & JT-CG & 0.0963 & 0.00 & 0.09 & O & 0.1521 & 0.00 & 0.07 & O\\
 &  & JT-MP & 0.0963 & 0.00 & 0.12 & O & 0.1521 & 0.00 & 0.09 & O\\
\addlinespace[1.5pt]
3 & transactions & BOCT & 0.0158 & 100.00 & 605.40 & T & 0.0858 & 100.00 & 609.06 & T\\
 &  & DL8.5 & 0.0158 & --- & 601.13 & T & --- & --- & --- & ---\\
 &  & GOSDT & --- & --- & 13.63 & M & 0.0158 & 0.00 & 3.48 & O\\
 &  & MurTree & 0.0158 & --- & 104.53 & O & 0.0158 & --- & 14.25 & O\\
 &  & STreeD & 0.0158 & 0.00 & 28.30 & O & 0.0158 & 0.00 & 8.93 & O\\
 &  & Branches & 0.0158 & 100.00 & 45.34 & M & 0.0158 & 0.00 & 0.98 & O\\
 &  & JT-LP & 0.0158 & 0.00 & 5.17 & O & 0.0158 & 0.00 & 5.20 & O\\
 &  & JT-CG & 0.0158 & 0.00 & 0.91 & O & 0.0158 & 0.00 & 0.55 & O\\
 &  & JT-MP & 0.0158 & 0.00 & 1.98 & O & 0.0158 & 0.00 & 1.75 & O\\
\addlinespace[1.5pt]
4 & avila & BOCT & 0.4179 & 100.00 & 600.19 & T & 0.5479 & 93.71 & 600.13 & T\\
 &  & DL8.5 & 0.3779 & --- & 600.03 & T & --- & --- & --- & ---\\
 &  & GOSDT & --- & --- & 15.18 & M & 0.4780 & 93.69 & 610.52 & T\\
 &  & MurTree & 0.3779 & --- & 50.19 & O & 0.4749 & --- & 64.46 & O\\
 &  & STreeD & 0.3779 & 0.00 & 65.93 & O & 0.4749 & 0.00 & 64.82 & O\\
 &  & Branches & 0.4176 & 100.00 & 111.79 & M & 0.4832 & 83.68 & 186.27 & M\\
 &  & JT-LP & 0.3779 & 0.00 & 1.51 & O & 0.4749 & 0.00 & 1.47 & O\\
 &  & JT-CG & 0.3779 & 0.00 & 1.57 & O & 0.4749 & 0.00 & 1.46 & O\\
 &  & JT-MP & 0.3779 & 0.00 & 1.75 & O & 0.4749 & 0.00 & 1.54 & O\\
\addlinespace[1.5pt]
4 & banknote & BOCT & 0.0219 & 100.00 & 600.03 & T & 0.1076 & 66.06 & 600.03 & T\\
 &  & DL8.5 & 0.0066 & --- & 1.12 & O & --- & --- & --- & ---\\
 &  & GOSDT & 0.0066 & 0.00 & 1.86 & O & 0.0784 & 0.00 & 1.24 & O\\
 &  & MurTree & 0.0066 & --- & 0.08 & O & 0.0784 & --- & 0.14 & O\\
 &  & STreeD & 0.0066 & 0.00 & 0.90 & O & 0.0784 & 0.00 & 0.90 & O\\
 &  & Branches & 0.0066 & 0.00 & 7.19 & O & 0.0784 & 0.00 & 2.51 & O\\
 &  & JT-LP & 0.0066 & 0.00 & 0.02 & O & 0.0784 & 0.00 & 0.03 & O\\
 &  & JT-CG & 0.0066 & 0.00 & 0.10 & O & 0.0784 & 0.00 & 0.09 & O\\
4 & banknote & JT-MP & 0.0066 & 0.00 & 0.05 & O & 0.0784 & 0.00 & 0.05 & O\\
\addlinespace[1.5pt]
4 & compas & BOCT & 0.2743 & 100.00 & 600.25 & T & 0.2980 & 94.75 & 600.09 & T\\
 &  & DL8.5 & 0.2575 & --- & 151.06 & O & --- & --- & --- & ---\\
 &  & GOSDT & 0.2575 & 0.00 & 217.06 & O & 0.2953 & 0.00 & 124.58 & O\\
 &  & MurTree & 0.2575 & --- & 14.32 & O & 0.2953 & --- & 16.37 & O\\
 &  & STreeD & 0.2575 & 0.00 & 3.92 & O & 0.2953 & 0.00 & 3.34 & O\\
 &  & Branches & 0.2650 & 100.00 & 92.48 & M & 0.2953 & 75.42 & 114.10 & M\\
4 & compas & JT-LP & 0.2575 & 0.00 & 0.40 & O & 0.2953 & 0.00 & 0.41 & O\\
 &  & JT-CG & 0.2575 & 0.00 & 0.68 & O & 0.2953 & 0.00 & 0.60 & O\\
 &  & JT-MP & 0.2575 & 0.00 & 0.60 & O & 0.2953 & 0.00 & 0.56 & O\\
\addlinespace[1.5pt]
4 & diabetic & BOCT & 0.4261 & 100.00 & 601.21 & NI & 0.5361 & 100.00 & 601.32 & NI\\
 &  & DL8.5 & 0.4331 & --- & 600.27 & T & --- & --- & --- & ---\\
 &  & GOSDT & --- & --- & 10.81 & M & --- & --- & 44.60 & M\\
 &  & MurTree & --- & --- & 604.19 & T & --- & --- & 603.73 & T\\
 &  & STreeD & 0.4212 & 0.00 & 504.89 & O & 0.4452 & 0.00 & 576.20 & O\\
 &  & Branches & 0.4262 & 100.00 & 193.60 & M & 0.4452 & 93.26 & 569.81 & M\\
 &  & JT-LP & 0.4212 & 0.00 & 212.63 & O & 0.4452 & 0.00 & 209.12 & O\\
 &  & JT-CG & 0.4212 & 0.00 & 162.33 & O & 0.4452 & 0.00 & 99.12 & O\\
 &  & JT-MP & 0.4212 & 0.00 & 185.62 & O & 0.4452 & 0.00 & 109.37 & O\\
\addlinespace[1.5pt]
4 & fico & BOCT & 0.2924 & 100.00 & 600.35 & T & 0.4424 & 95.48 & 600.13 & T\\
 &  & DL8.5 & 0.2784 & --- & 600.02 & T & --- & --- & --- & ---\\
 &  & GOSDT & --- & --- & 298.03 & M & --- & --- & 344.32 & M\\
 &  & MurTree & 0.2760 & --- & 187.16 & O & 0.3173 & --- & 183.42 & O\\
 &  & STreeD & 0.2760 & 0.00 & 64.70 & O & 0.3173 & 0.00 & 61.06 & O\\
 &  & Branches & 0.2874 & 100.00 & 106.45 & M & 0.3173 & 81.06 & 122.53 & M\\
 &  & JT-LP & 0.2760 & 0.00 & 3.10 & O & 0.3173 & 0.00 & 3.05 & O\\
 &  & JT-CG & 0.2760 & 0.00 & 2.75 & O & 0.3173 & 0.00 & 2.22 & O\\
 &  & JT-MP & 0.2760 & 0.00 & 2.79 & O & 0.3173 & 0.00 & 2.23 & O\\
\addlinespace[1.5pt]
4 & give & BOCT & 0.0642 & 100.00 & 603.10 & T & 0.0768 & 85.18 & 600.53 & T\\
 &  & DL8.5 & 0.0637 & --- & 600.18 & T & --- & --- & --- & ---\\
 &  & GOSDT & --- & --- & 16.98 & M & 0.0668 & 0.00 & 24.71 & O\\
 &  & MurTree & 0.0636 & --- & 89.51 & O & 0.0668 & --- & 41.72 & O\\
 &  & STreeD & 0.0636 & 0.00 & 38.53 & O & 0.0668 & 0.00 & 22.61 & O\\
 &  & Branches & 0.0668 & 100.00 & 54.16 & M & 0.0668 & 12.71 & 105.30 & M\\
 &  & JT-LP & 0.0636 & 0.00 & 0.76 & O & 0.0668 & 0.00 & 0.77 & O\\
 &  & JT-CG & 0.0636 & 0.00 & 1.40 & O & 0.0668 & 0.00 & 1.08 & O\\
 &  & JT-MP & 0.0636 & 0.00 & 1.44 & O & 0.0668 & 0.00 & 1.01 & O\\
\addlinespace[1.5pt]
4 & htru2 & BOCT & 0.0219 & 100.00 & 600.78 & T & 0.0422 & 52.65 & 600.09 & T\\
 &  & DL8.5 & 0.0210 & --- & 173.90 & O & --- & --- & --- & ---\\
 &  & GOSDT & 0.0210 & 0.00 & 219.51 & O & 0.0422 & 0.00 & 0.41 & O\\
 &  & MurTree & 0.0210 & --- & 13.73 & O & 0.0422 & --- & 1.29 & O\\
 &  & STreeD & 0.0210 & 0.00 & 5.34 & O & 0.0422 & 0.00 & 1.00 & O\\
 &  & Branches & 0.0215 & 100.00 & 43.69 & M & 0.0422 & 0.00 & 15.95 & O\\
 &  & JT-LP & 0.0210 & 0.00 & 0.42 & O & 0.0422 & 0.00 & 0.39 & O\\
 &  & JT-CG & 0.0210 & 0.00 & 0.69 & O & 0.0422 & 0.00 & 0.55 & O\\
 &  & JT-MP & 0.0210 & 0.00 & 0.60 & O & 0.0422 & 0.00 & 0.48 & O\\
\addlinespace[1.5pt]
4 & letter & BOCT & 0.6895 & 100.00 & 600.21 & T & 0.8396 & 96.25 & 600.17 & T\\
4 & letter & DL8.5 & 0.7035 & --- & 600.03 & T & --- & --- & --- & ---\\
 &  & GOSDT & --- & --- & 17.99 & M & 0.8877 & 96.61 & 643.88 & T\\
 &  & MurTree & 0.6044 & --- & 120.63 & O & 0.7529 & --- & 151.59 & O\\
 &  & STreeD & 0.6044 & 0.00 & 172.82 & O & 0.7529 & 0.00 & 204.25 & O\\
 &  & Branches & 0.7136 & 100.00 & 122.11 & M & 0.8192 & 91.02 & 199.00 & M\\
 &  & JT-LP & 0.6044 & 0.00 & 5.11 & O & 0.7529 & 0.00 & 5.38 & O\\
 &  & JT-CG & 0.6044 & 0.00 & 2.62 & O & 0.7529 & 0.00 & 2.68 & O\\
4 & letter & JT-MP & 0.6044 & 0.00 & 3.42 & O & 0.7529 & 0.00 & 3.21 & O\\
\addlinespace[1.5pt]
4 & skin & BOCT & 0.0264 & 100.00 & 600.96 & T & 0.1464 & 100.00 & 601.01 & NI\\
 &  & DL8.5 & 0.0173 & --- & 11.29 & O & --- & --- & --- & ---\\
 &  & GOSDT & 0.0173 & 0.00 & 45.63 & O & 0.0861 & 0.00 & 23.00 & O\\
 &  & MurTree & 0.0173 & --- & 6.45 & O & 0.0861 & --- & 8.02 & O\\
 &  & STreeD & 0.0173 & 0.00 & 4.67 & O & 0.0861 & 0.00 & 5.79 & O\\
 &  & Branches & 0.0173 & 0.00 & 20.51 & O & 0.0861 & 0.00 & 24.68 & O\\
 &  & JT-LP & 0.0173 & 0.00 & 0.32 & O & 0.0861 & 0.00 & 0.34 & O\\
 &  & JT-CG & 0.0173 & 0.00 & 0.80 & O & 0.0861 & 0.00 & 0.71 & O\\
 &  & JT-MP & 0.0173 & 0.00 & 0.70 & O & 0.0861 & 0.00 & 0.68 & O\\
\addlinespace[1.5pt]
4 & spambase & BOCT & 0.0930 & 99.30 & 600.06 & T & 0.2217 & 89.78 & 600.48 & T\\
 &  & DL8.5 & 0.0859 & --- & 600.01 & T & --- & --- & --- & ---\\
 &  & GOSDT & 0.1300 & 98.16 & 606.09 & T & 0.1791 & 87.38 & 606.93 & T\\
 &  & MurTree & 0.0804 & --- & 70.92 & O & 0.1489 & --- & 100.24 & O\\
 &  & STreeD & 0.0804 & 0.00 & 10.48 & O & 0.1489 & 0.00 & 10.26 & O\\
 &  & Branches & 0.1758 & 100.00 & 93.12 & M & 0.1528 & 64.58 & 175.76 & M\\
 &  & JT-LP & 0.0804 & 0.00 & 1.87 & O & 0.1489 & 0.00 & 1.83 & O\\
 &  & JT-CG & 0.0804 & 0.00 & 1.72 & O & 0.1489 & 0.00 & 1.26 & O\\
 &  & JT-MP & 0.0804 & 0.00 & 1.96 & O & 0.1489 & 0.00 & 1.29 & O\\
\addlinespace[1.5pt]
4 & transactions & BOCT & 0.0158 & 100.00 & 609.53 & T & 0.1658 & 100.00 & 610.26 & T\\
 &  & DL8.5 & 0.0158 & --- & 601.87 & T & --- & --- & --- & ---\\
 &  & GOSDT & --- & --- & 12.62 & M & 0.0158 & 0.00 & 3.59 & O\\
 &  & MurTree & --- & --- & 613.27 & T & 0.0158 & --- & 15.03 & O\\
 &  & STreeD & 0.0158 & 100.00 & 600.00 & T & 0.0158 & 0.00 & 8.94 & O\\
 &  & Branches & 0.0158 & 100.00 & 36.42 & M & 0.0158 & 0.00 & 0.94 & O\\
 &  & JT-LP & 0.0158 & 0.00 & 77.80 & O & 0.0158 & 0.00 & 76.71 & O\\
 &  & JT-CG & 0.0158 & 0.00 & 51.81 & O & 0.0158 & 0.00 & 0.20 & O\\
 &  & JT-MP & 0.0158 & 0.00 & 59.79 & O & 0.0158 & 0.00 & 0.41 & O\\
\addlinespace[1.5pt]
5 & avila & BOCT & 0.3850 & 100.00 & 600.95 & T & 0.6150 & 94.42 & 600.13 & T\\
 &  & DL8.5 & 0.3960 & --- & 600.03 & T & --- & --- & --- & ---\\
 &  & GOSDT & --- & --- & 14.63 & M & --- & --- & 116.24 & M\\
 &  & MurTree & --- & --- & 599.71 & T & --- & --- & 599.95 & T\\
 &  & STreeD & 0.3279 & 100.00 & 599.06 & T & 0.4749 & 100.00 & 599.86 & T\\
 &  & Branches & 0.4723 & 100.00 & 97.23 & M & 0.4832 & 83.81 & 200.96 & M\\
 &  & JT-LP & 0.3279 & 0.00 & 77.76 & O & 0.4749 & 0.00 & 77.72 & O\\
 &  & JT-CG & 0.3279 & 0.00 & 54.12 & O & 0.4749 & 0.00 & 39.61 & O\\
 &  & JT-MP & 0.3279 & 0.00 & 79.44 & O & 0.4749 & 0.00 & 51.09 & O\\
\addlinespace[1.5pt]
5 & banknote & BOCT & 0.0102 & 100.00 & 600.04 & T & 0.1294 & 75.05 & 600.03 & T\\
 &  & DL8.5 & 0.0000 & --- & 2.21 & O & --- & --- & --- & ---\\
 &  & GOSDT & 0.0000 & 0.00 & 20.42 & O & 0.0784 & 0.00 & 10.10 & O\\
5 & banknote & MurTree & 0.0000 & --- & 0.14 & O & 0.0784 & --- & 0.80 & O\\
 &  & STreeD & 0.0000 & 0.00 & 0.90 & O & 0.0784 & 0.00 & 1.18 & O\\
 &  & Branches & 0.0000 & 0.00 & 25.22 & O & 0.0784 & 0.00 & 14.40 & O\\
 &  & JT-LP & 0.0000 & 0.00 & 0.20 & O & 0.0784 & 0.00 & 0.23 & O\\
 &  & JT-CG & 0.0000 & 0.00 & 0.29 & O & 0.0784 & 0.00 & 0.65 & O\\
 &  & JT-MP & 0.0000 & 0.00 & 0.04 & O & 0.0784 & 0.00 & 0.41 & O\\
\addlinespace[1.5pt]
5 & compas & BOCT & 0.2642 & 100.00 & 600.28 & T & 0.2980 & 95.19 & 600.14 & T\\
5 & compas & DL8.5 & 0.2506 & --- & 600.02 & T & --- & --- & --- & ---\\
 &  & GOSDT & --- & --- & 152.61 & M & --- & --- & 207.23 & M\\
 &  & MurTree & 0.2498 & --- & 325.91 & O & 0.2953 & --- & 414.00 & O\\
 &  & STreeD & 0.2498 & 0.00 & 67.81 & O & 0.2953 & 0.00 & 54.25 & O\\
 &  & Branches & 0.2632 & 100.00 & 92.68 & M & 0.2953 & 76.13 & 127.06 & M\\
 &  & JT-LP & 0.2498 & 0.00 & 6.41 & O & 0.2953 & 0.00 & 6.55 & O\\
 &  & JT-CG & 0.2498 & 0.00 & 9.23 & O & 0.2953 & 0.00 & 6.50 & O\\
 &  & JT-MP & 0.2498 & 0.00 & 8.49 & O & 0.2953 & 0.00 & 4.99 & O\\
\addlinespace[1.5pt]
5 & diabetic & BOCT & 0.4254 & 100.00 & 601.75 & T & 0.6354 & 100.00 & 602.09 & T\\
 &  & DL8.5 & 0.4323 & --- & 600.20 & T & --- & --- & --- & ---\\
 &  & GOSDT & --- & --- & 11.05 & M & --- & --- & 43.33 & M\\
 &  & MurTree & --- & --- & 604.02 & T & --- & --- & 603.66 & T\\
 &  & STreeD & 0.4609 & 100.00 & 600.00 & T & 0.4609 & 100.00 & 600.00 & T\\
 &  & Branches & 0.4262 & 100.00 & 95.64 & M & 0.4452 & 93.26 & 572.46 & M\\
 &  & JT-LP & 0.4252 & 100.00 & 600.28 & T & 0.4452 & 100.00 & 600.44 & T\\
 &  & JT-CG & 0.4184 & 99.99 & 610.03 & T & 0.4452 & 91.00 & 606.44 & T\\
 &  & JT-MP & 0.4168 & 99.99 & 599.99 & T & 0.4452 & 91.01 & 599.99 & T\\
\addlinespace[1.5pt]
5 & fico & BOCT & 0.2810 & 100.00 & 600.80 & T & 0.3568 & 95.80 & 600.49 & T\\
 &  & DL8.5 & 0.2784 & --- & 600.02 & T & --- & --- & --- & ---\\
 &  & GOSDT & --- & --- & 68.69 & M & --- & --- & 70.48 & M\\
 &  & MurTree & --- & --- & 599.48 & T & --- & --- & 599.51 & T\\
 &  & STreeD & 0.2658 & 100.00 & 599.74 & T & 0.3173 & 100.00 & 599.96 & T\\
 &  & Branches & 0.3016 & 100.00 & 100.93 & M & 0.3173 & 81.06 & 128.88 & M\\
 &  & JT-LP & 0.2649 & 0.00 & 348.02 & O & 0.3173 & 0.00 & 349.30 & O\\
 &  & JT-CG & 0.2649 & 0.00 & 256.61 & O & 0.3173 & 0.00 & 156.38 & O\\
 &  & JT-MP & 0.2649 & 0.00 & 300.96 & O & 0.3173 & 0.00 & 173.08 & O\\
\addlinespace[1.5pt]
5 & give & BOCT & 0.0637 & 100.00 & 603.26 & NI & 0.3737 & 100.00 & 600.64 & NI\\
 &  & DL8.5 & 0.0635 & --- & 600.25 & T & --- & --- & --- & ---\\
 &  & GOSDT & --- & --- & 16.70 & M & 0.0668 & 0.00 & 23.56 & O\\
 &  & MurTree & --- & --- & 600.66 & T & 0.0668 & --- & 443.75 & O\\
 &  & STreeD & 0.0668 & 100.00 & 600.00 & T & 0.0668 & 0.00 & 174.18 & O\\
 &  & Branches & 0.0668 & 100.00 & 46.94 & M & 0.0668 & 13.83 & 117.52 & M\\
 &  & JT-LP & 0.0629 & 0.00 & 19.96 & O & 0.0668 & 0.00 & 19.43 & O\\
 &  & JT-CG & 0.0629 & 0.00 & 29.30 & O & 0.0668 & 0.00 & 11.06 & O\\
 &  & JT-MP & 0.0629 & 0.00 & 36.25 & O & 0.0668 & 0.00 & 11.97 & O\\
\addlinespace[1.5pt]
5 & htru2 & BOCT & 0.0213 & 100.00 & 601.32 & T & 0.0422 & 52.38 & 600.13 & T\\
 &  & DL8.5 & 0.0200 & --- & 600.03 & T & --- & --- & --- & ---\\
 &  & GOSDT & --- & --- & 128.24 & M & 0.0422 & 0.00 & 0.40 & O\\
 &  & MurTree & 0.0197 & --- & 240.39 & O & 0.0422 & --- & 1.27 & O\\
 &  & STreeD & 0.0197 & 0.00 & 78.37 & O & 0.0422 & 0.00 & 1.01 & O\\
5 & htru2 & Branches & 0.0217 & 100.00 & 91.39 & M & 0.0422 & 0.00 & 19.10 & O\\
 &  & JT-LP & 0.0197 & 0.00 & 5.58 & O & 0.0422 & 0.00 & 5.53 & O\\
 &  & JT-CG & 0.0197 & 0.00 & 8.17 & O & 0.0422 & 0.00 & 1.11 & O\\
 &  & JT-MP & 0.0197 & 0.00 & 6.74 & O & 0.0422 & 0.00 & 0.94 & O\\
\addlinespace[1.5pt]
5 & letter & BOCT & 0.5477 & 100.00 & 601.67 & T & 0.8578 & 96.36 & 600.16 & T\\
 &  & DL8.5 & 0.7575 & --- & 600.04 & T & --- & --- & --- & ---\\
 &  & GOSDT & --- & --- & 16.42 & M & --- & --- & 134.85 & M\\
5 & letter & MurTree & --- & --- & 599.52 & T & --- & --- & 599.66 & T\\
 &  & STreeD & 0.4854 & 100.00 & 599.19 & T & 0.7417 & 100.00 & 599.45 & T\\
 &  & Branches & 0.8219 & 100.00 & 108.09 & M & 0.8192 & 91.19 & 201.24 & M\\
 &  & JT-LP & 0.4425 & 0.00 & 358.86 & O & 0.7111 & 0.00 & 357.08 & O\\
 &  & JT-CG & 0.4425 & 0.00 & 76.28 & O & 0.7111 & 0.00 & 69.44 & O\\
 &  & JT-MP & 0.4425 & 0.00 & 118.59 & O & 0.7111 & 0.00 & 103.36 & O\\
\addlinespace[1.5pt]
5 & skin & BOCT & 0.0239 & 100.00 & 600.96 & NI & 0.2239 & 100.00 & 600.54 & NI\\
 &  & DL8.5 & 0.0115 & --- & 132.51 & O & --- & --- & --- & ---\\
 &  & GOSDT & 0.0115 & 0.00 & 289.17 & O & 0.0861 & 0.00 & 99.12 & O\\
 &  & MurTree & 0.0115 & --- & 32.89 & O & 0.0861 & --- & 34.70 & O\\
 &  & STreeD & 0.0115 & 0.00 & 27.90 & O & 0.0861 & 0.00 & 29.65 & O\\
 &  & Branches & 0.0225 & 100.00 & 39.60 & M & 0.0861 & 18.71 & 55.16 & M\\
 &  & JT-LP & 0.0115 & 0.00 & 0.78 & O & 0.0861 & 0.00 & 0.75 & O\\
 &  & JT-CG & 0.0115 & 0.00 & 2.03 & O & 0.0861 & 0.00 & 1.71 & O\\
 &  & JT-MP & 0.0115 & 0.00 & 2.02 & O & 0.0861 & 0.00 & 1.56 & O\\
\addlinespace[1.5pt]
5 & spambase & BOCT & 0.0826 & 99.47 & 600.31 & T & 0.2217 & 90.15 & 600.28 & T\\
 &  & DL8.5 & 0.2630 & --- & 600.01 & T & --- & --- & --- & ---\\
 &  & GOSDT & --- & --- & 76.62 & M & --- & --- & 89.44 & M\\
 &  & MurTree & --- & --- & 599.42 & T & --- & --- & 599.44 & T\\
 &  & STreeD & 0.0635 & 0.00 & 418.99 & O & 0.1469 & 0.00 & 448.60 & O\\
 &  & Branches & 0.1756 & 100.00 & 148.54 & M & 0.1528 & 65.00 & 175.87 & M\\
 &  & JT-LP & 0.0635 & 0.00 & 181.97 & O & 0.1469 & 0.00 & 182.60 & O\\
 &  & JT-CG & 0.0635 & 0.00 & 123.05 & O & 0.1469 & 0.00 & 75.93 & O\\
 &  & JT-MP & 0.0635 & 0.00 & 171.02 & O & 0.1469 & 0.00 & 71.81 & O\\
\addlinespace[1.5pt]
5 & transactions & BOCT & 0.0158 & 100.00 & 611.48 & T & 0.3158 & 100.00 & 611.66 & T\\
 &  & DL8.5 & 0.0158 & --- & 601.96 & T & --- & --- & --- & ---\\
 &  & GOSDT & --- & --- & 12.59 & M & 0.0158 & 0.00 & 3.40 & O\\
 &  & MurTree & --- & --- & 613.20 & T & 0.0158 & --- & 14.34 & O\\
 &  & STreeD & 0.0158 & 100.00 & 600.00 & T & 0.0158 & 0.00 & 9.00 & O\\
 &  & Branches & 0.0158 & 100.00 & 23.81 & M & 0.0158 & 0.00 & 0.94 & O\\
 &  & JT-LP & 0.0158 & 100.00 & 621.06 & T & 0.0158 & 100.00 & 611.95 & T\\
 &  & JT-CG & 0.0158 & 99.92 & 601.47 & T & 0.0158 & 0.00 & 0.20 & O\\
 &  & JT-MP & 0.0157 & 99.92 & 600.53 & T & 0.0158 & 0.00 & 0.39 & O\\

\end{longtable}
\endgroup
% ===== END inlined file: tables/benchmark_detail =====

\clearpage
\subsection{Unresolved deep settings}\label{sec:incomplete}
\begin{table}[!htbp]
\centering\small
\caption{Settings not certified by both JT-CG and STreeD within 600 seconds. Bounds are JT-CG bounds unless marked otherwise.}\label{tab:ecincomplete}
\begin{tabular}{lrrrll}
\toprule
Dataset & $D$ & $\lambda$ & JT-CG status & JT-CG bound/optimum & STreeD status\\
\midrule
avila & 5 & 0 & O & 0.327886 & T\\
avila & 5 & .01 & O & 0.474929 & T\\
diabetic & 5 & 0 & T & $[0.000029,0.418372]$ & T\\
diabetic & 5 & .01 & T & $[0.040069,0.445204]$ & T\\
fico & 5 & 0 & O & 0.264939 & T\\
fico & 5 & .01 & O & 0.317256 & T\\
give & 5 & 0 & O & 0.062947 & T\\
letter & 5 & 0 & O & 0.442500 & T\\
letter & 5 & .01 & O & 0.711100 & T\\
transactions & 4 & 0 & O & 0.015771 & T\\
transactions & 5 & 0 & T & $[0.000013,0.015756]$ & T\\
\bottomrule
\end{tabular}
\end{table}

{
Figure~\ref{fig:difficult-trajectories} traces the three uncertified JT-CG
settings. The final feasible objectives are first obtained within 3.5 seconds.
Subsequent evaluation neither improves these incumbents nor closes the gap;
the final lower bounds remain well below the corresponding upper bounds in
Table~\ref{tab:ecincomplete}.
\par}

\begin{figure}[!htbp]
\centering
\includegraphics[width=\linewidth]{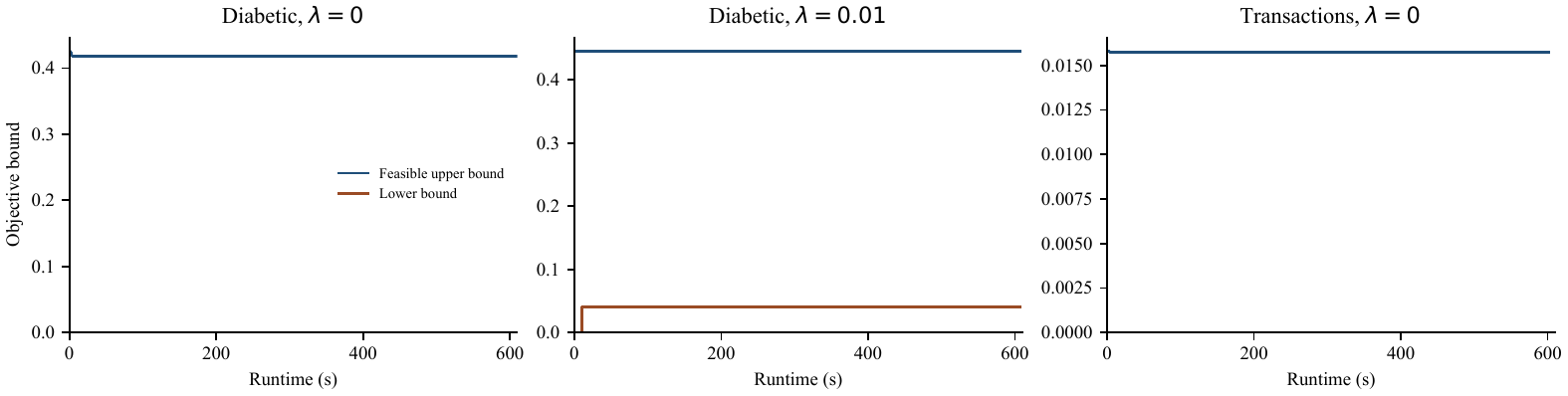}
\caption{{Observed lower bounds and feasible upper bounds for
the three uncertified JT-CG settings. Curves show recorded updates only.}}
\label{fig:difficult-trajectories}
\end{figure}

\end{document}